\documentclass[11pt]{article}
\usepackage{fullpage}

\def\notes{1}
\ifnum\notes=1
\newcommand{\mynote}[3]{\marginpar{\tiny \sf \color{#1} {#2}: {#3}}}
\else
\newcommand{\mynote}[3]{}
\fi

\usepackage{multirow}

\usepackage[style=alphabetic,backend=bibtex,maxnames=12,maxbibnames=10,maxcitenames=10,maxalphanames=10,giveninits=true,doi=false,url=true]{biblatex}
\newcommand*{\citet}[1]{\AtNextCite{\AtEachCitekey{\defcounter{maxnames}{2}}} \textcite{#1}}
\newcommand*{\citetall}[1]{\AtNextCite{\AtEachCitekey{\defcounter{maxnames}{999}}} \textcite{#1}}
\newcommand*{\citep}[1]{\cite{#1}}

\bibliography{biblio}

\usepackage[utf8]{inputenc} % allow utf-8 input
\usepackage[T1]{fontenc}    % use 8-bit T1 fonts
\usepackage{hyperref}       % hyperlinks
\usepackage{url}            % simple URL typesetting
\usepackage{booktabs}       % professional-quality tables
\usepackage{amsfonts}       % blackboard math symbols
\usepackage{nicefrac}       % compact symbols for 1/2, etc.
\usepackage{microtype}      % microtypography
\usepackage{subcaption}
\usepackage{amsmath, amssymb, amsthm}
\usepackage{bbm}
\usepackage{enumerate}
\usepackage{framed}
\usepackage{indentfirst}
\usepackage{paralist}
\usepackage{verbatim}
\usepackage{microtype}
\usepackage{xspace}
\usepackage{wrapfig}
\usepackage{thm-restate}
\usepackage{multirow}
\usepackage{hhline}
\usepackage{amsmath}
\usepackage{biblatex}

\usepackage[usenames,dvipsnames]{xcolor}
	\definecolor{DarkGreen}{rgb}{0.15,0.5,0.15}
	\definecolor{DarkRed}{rgb}{0.6,0.2,0.2}
	\definecolor{DarkBlue}{rgb}{0.2,0.2,0.6}
	\definecolor{DarkPurple}{rgb}{0.4,0.2,0.4}
\usepackage[ruled,vlined]{algorithm2e}

\usepackage{url}
\usepackage{hyperref}
	\hypersetup{
		linktocpage=true,
		colorlinks=true,				
		linkcolor=DarkBlue,		
		citecolor=DarkBlue,	
		urlcolor=DarkBlue,
	}

\def\output{\mathcal{S}}
\def\lr{\mathcal{R}}
\def\lo{\textup{\texttt{LO}}}
\def\shuffler{\mathcal{A}_{\rm s}}
\def\bin{{\rm Bin}}
\def\mnom{{\rm MultNom}}
\def\TV{{\rm TV}}

\def\Q{\mathcal{Q}}
\def\A{\mathcal{A}}
\def\L{\mathcal{L}}
\def\D{\mathcal{D}}
\def\Aldp[#1]{\mathcal{R}^{(#1)}}
\def\AldpNI{{\mathcal{A}_{\rm ldp}}}
\def\out[#1]{\mathcal{S}^{(#1)}}
\def\2RR{\texttt{2RR}}
\def\kRR{\texttt{kRR}}
\def\reference{\rho}
\def\indicator{\mathds{1}}
\def\BinS{\texttt{BinS}}

\newcommand{\dalpha}[1]{D_{#1}}
\newcommand{\onedint}{\mathcal{B}}
\newcommand{\zo}{\{0,1\}}
\newcommand{\zot}{\{0,1,2\}}
\newcommand{\zotf}{\{0,1,2,3\}}

\providecommand{\equ}[1]{

\begin{equation}
#1
\end{equation}}

\providecommand{\alequn}[1]{\begin{align*} #1 \end{align*}}

\newcommand{\new}[1]{\textcolor{black}{#1}}

\let\originalleft\left
\let\originalright\right
\renewcommand{\left}{\mathopen{}\mathclose\bgroup\originalleft}
\renewcommand{\right}{\aftergroup\egroup\originalright}

\newcommand{\ex}[2]{{\ifx&#1& \mathbb{E} \else \underset{#1}{\mathbb{E}} \fi \left[#2\right]}}
\newcommand{\pr}[2]{{\ifx&#1& \mathbb{P} \else \underset{#1}{\mathbb{P}} \fi \left[#2\right]}}

\newcommand{\eps}{{\varepsilon}}

\newcommand{\Ber}{\ensuremath{\mathrm{Bern}}}

\usepackage{times}

\newcommand{\cD}{\mathcal{D}}

\newcommand{\cS}{\mathcal{S}}

\newtheorem{thm}{Theorem}[section]
\newtheorem{theorem}[thm]{Theorem}

\newtheorem{lemma}[thm]{Lemma}

\newtheorem{prop}[thm]{Proposition}

	\theoremstyle{definition}

\newtheorem{definition}[thm]{Definition}

\newcommand{\E}{\mathbb{E}}

\newcommand{\Bin}{\operatorname{Bin}}

\newcommand{\unif}[1]{\mathcal{U}_{#1}}

\usepackage{graphicx}
\usepackage{dsfont}

\title{Hiding Among the Clones: A Simple and Nearly Optimal Analysis of Privacy Amplification by Shuffling}

\author{
  Vitaly Feldman\\
  Apple\\
  \and {Audra McMillan}\\
  Apple\\
  \texttt{audra\textunderscore mcmillan@apple.com}\\
  \and Kunal Talwar\\
  Apple\\
  \texttt{ktalwar@apple.com}\\
}
\date{}

\begin{document}

\maketitle

\begin{abstract}
  Recent work of \citetall{ErlingssonFMRTT19} demonstrates that random shuffling  amplifies differential privacy guarantees of locally randomized data. Such amplification implies substantially stronger privacy guarantees for systems in which data is contributed anonymously \cite{Bittau17} and has lead to significant interest in the shuffle model of privacy \cite{Cheu:2019,ErlingssonFMRTT19}.

   We show that random shuffling of $n$ data records that are input to $\eps_0$-differentially private local randomizers results in an $(O((1-e^{-\eps_0})\sqrt{\frac{e^{\eps_0}\log(1/\delta)}{n}}), \delta)$-differentially private algorithm. This significantly improves over previous work and achieves the asymptotically optimal dependence in $\eps_0$. Our result is based on a new approach that is simpler than previous work and extends to approximate differential privacy with nearly the same guarantees. Importantly, our work also yields an algorithm for deriving tighter bounds on the resulting $\eps$ and $\delta$ as well as R\'enyi differential privacy guarantees. We show numerically that our algorithm gets to within a small constant factor of the optimal bound. As a direct corollary of our analysis we derive a simple and nearly optimal algorithm for frequency estimation in the shuffle model of privacy. We also observe that our result implies the first asymptotically optimal privacy analysis of noisy stochastic gradient descent that applies to sampling without replacement.
\end{abstract}

\thispagestyle{empty}
\newpage
\setcounter{page}{1}
\section{Introduction}
We consider privacy-preserving data analysis in the local model of differential privacy augmented with a shuffler.
In this model, each user sends a locally differentially private report and  these reports are then anonymized and randomly shuffled. Systems based on this model were first proposed in \cite{Bittau17}. The authors of \cite{ErlingssonFMRTT19} showed that random shuffling of inputs to locally private protocols amplifies the privacy guarantee. Thus,
when the collection of anonymized reports is viewed in the central model, the privacy guarantees are substantially stronger than the original local privacy guarantees. A similar result was shown for the binary randomized response by \citetall{Cheu:2019} who also formalized a related {\em shuffle model of privacy}.

The analysis in \cite{ErlingssonFMRTT19} relies on a more general result referred to as {\em privacy amplification by shuffling}. This result shows that privacy is amplified when the inputs are shuffled {\em before} applying local randomizers and holds even when local randomizers are chosen sequentially and adaptively.
Allowing adaptive choice of local randomizers is necessary for analyzing iterative optimization algorithms such as stochastic gradient descent.

A key limitation of the amplification result in \cite{ErlingssonFMRTT19} is that it has sub-optimal dependence on the local privacy parameter $\eps_0$ when $\eps_0 > 1$. Specifically, in this regime the resulting central privacy parameter is $\tilde O(e^{3\eps_0}/\sqrt{n})$\footnote{The dependence on $\eps_0$ was recently sharpened to $e^{2.5\eps_0}$ in~\cite{Balle:2020}.}. In addition, the result in \cite{ErlingssonFMRTT19} does not apply to local randomizers that satisfy only approximate differential privacy. For the simpler setting in which the local randomizer is fixed (and, in particular, non-adaptive)~\citetall{Balle:2019} give a stronger bound of $\tilde O(e^{\eps_0}/\sqrt{n})$ using a different proof approach. These results are in contrast to the binary randomized response for which the bound is just $\tilde O(e^{\eps_0/2}/\sqrt{n})$ \cite{Cheu:2019}.
This change in the exponent is significant in the regime where $\eps_0 > 1$, which is commonly used in practice. Further, this regime naturally arises when the target central $\eps$ is greater than $1/\sqrt{n}$.

In this paper we give a new analysis of privacy amplification by shuffling.
Specifically, we show that running an adaptive sequence of arbitrary $\eps_0$-DP local randomizers on a uniformly random permutation of $n$ data items, yields an $(\eps, \delta)$-DP algorithm, where \[\eps = O\left((1-e^{-\eps_0})\frac{\sqrt{e^{\eps_0}\log(1/\delta)}}{\sqrt{n}}\right).\] When $\eps_0 > 1$, this improves the dependence on $\eps_0$ from $e^{2.5\eps_0}$ to asymptotically correct dependence of $e^{\eps_0/2}$ which was previously known only for the binary randomized response~\cite{Cheu:2019}.
The bound  matches the existing bounds in the regime when $\eps_0<1$ and is asymptotically optimal in both regimes.
  We then extend this analysis to approximately differentially private local randomizers with essentially the same guarantees. The best previous guarantee in this case has dependence of $e^{20\eps_0}$ given in \cite{Balle:2020}.

Importantly, our proof is also simpler than the rather delicate approaches in prior work \cite{ErlingssonFMRTT19,Balle:2019}.
The proof reduces the problem of analyzing the shuffling privacy guarantee of an adaptive series of local algorithms to analyzing the shuffling privacy guarantee of a simple non-adaptive local algorithm with three outputs. Intuitively, we argue that for any local randomizer $\mathcal A$ and two data points $x,y$, $\A(y)$ can be seen as sampling from the same distribution as $\A(x)$ with some positive probability. That is, each data point can create a clone of the output of $\A(x)$ with some probability. For the $\eps_0 > 1$ regime, the desired privacy guarantees then follow easily from the number of clones that $x$ has to hide among being distributed as a binomial random variable. In the small $\eps_0$-regime we also rely on the fact that a local randomizer on two inputs can be seen as the composition of binary randomized response with a post-processing step \cite{KairouzOVb}. Another important property of our proof is that it provides a simple and efficient method for numerically computing an amplification bound that is tighter than our closed-form bounds \footnote{The resulting code is available at \href{https://github.com/apple/ml-shuffling-amplification}{https://github.com/apple/ml-shuffling-amplification}.}. Specifically, we show that it is sufficient to analyze the appropriate notion of divergence between two specific distributions over 3 values. \new{As in the case of binary randomised response, the constant 1/2 in the exponent arises naturally from analysing the divergence between binomial distributions.}
In Section~\ref{experiment} we show that this approach leads to numerical bounds that significantly improve on the state of the art. Our approach can be similarly used to numerically compute the privacy amplification guarantee in terms of other notions of privacy. In particular, we use it to compute the privacy amplification guarantee in terms of R\'enyi differential privacy. As we demonstrate in Section~\ref{renyiDPexperiments}, this results in tighter privacy guarantees when composing shuffled mechanisms.

Our proof extends naturally to approximate differential privacy. The extension to approximate DP involves an additional technical lemma (Lemma~\ref{approxtopure}) that states that any {\em deletion} (as opposed to replacement) $(\eps_0,\delta_0)$-DP local randomizer $\mathcal{A}$ can be converted to a deletion $\eps_0$-DP local randomizer $\mathcal{A}'$ such that $\TV(\mathcal{A}(x), \mathcal{A}'(x))\le\delta_0$ for all $x$.

Our privacy amplification result with the optimal dependence on $\eps_0$ has some surprising consequences. In certain important settings, it allows us to take existing and existing LDP and immediately obtain an algorithm that achieves (nearly) optimal trade-offs between privacy and utility {\em simultaneously} in the central {\em and} the local models of privacy. While the local $\eps_0$ may sometimes be large (we'll often need $\eps_0$ to grow logarithmically with $n$ to achieve a desired constant central $\eps$ guarantee), it provides an additional layer of protection that does not exist in the central model. Moreover, this shuffling result implies that a secure implementation of shuffling, or any function that is a post-processing of shuffling (e.g. vector summation), suffices to ensure the strong central privacy guarantee without having to assume a trusted curator. This is a considerably simpler task than using secure multiparty computation for computing the output of the central DP algorithm~\cite{ODOpaper}. Additionally, the fact that the privacy depends only on the LDP property of the local randomizer, and not on the specific noise distribution, allows us to do arbitrary post-processing (e.g. rounding, lossy compression) to the LDP responses before they are shuffled. The privacy of the aggregate relies only on the shuffling result, and we do not need to analyze the effect of the rounding/truncation as in several previous works that use distributed noise addition for a similar goal.

In Section~\ref{optimal}, we show an example of such an application for the problem of building a histogram (or estimating the frequency  distribution) over an alphabet of size $k$.  For a given level of accuracy, this algorithm simultaneously achieves nearly optimal central as well as local differential privacy bounds. The algorithm is based on applying our results to the algorithm from \citep{Acharya:2019} and has low communication and server-side decoding time. We remark that getting the correct bound here requires taking $\eps_0$ up to $\ln k$. Thus using the $e^{\eps_0}/\sqrt{n}$ amplification bound from \cite{Balle:2019} would give a result that is off by a factor of $\sqrt{k}$ in terms of the central privacy-utility tradeoff. In particular, the correct dependence on $\eps_0$ in the exponent is crucial to get this bound.

A second example in Section~\ref{sec:sgd} relies on the fact the our results holds for an adaptive sequence of local queries. We analyze the privacy of noisy stochastic gradient descent when run on a random permutation of the data. This bridges a disconnect between the prior theoretical analyses that assumed random sampling with replacement, and the practical implementations that use permutations.

Finally, we give an alternative analysis of privacy amplification by shuffling that can be tighter for specific LDP randomizers. In particular, we can use this to give a tighter analysis of $k$-ary randomized response. \new{A closed form amplification bound tailored to $k$-ary randomised response was also given in \cite{Balle:2019}. Their bound matches the bound presented in this work when $k=O(e^{\eps_0})$ and $\eps_0>1$, but is worse than our bound when $k=\Omega(e^{\eps_0})$. To the best of our knowledge, the bound presented in this work is the first closed form result that shows that the privacy amplification, for a fixed $\eps_0$, improves with $k$. This corroborates empirical results from \cite{Balle:2019}.}

\begin{table}[]
    \centering

    \begin{tabular}{| c | c | c|}
    \hline
        {\bf Algorithm/Setting} & {\bf Previous results} & {\bf This work}\\
    \hline\hline
        Binary randomized response &  $\frac{e^{\eps_0/2}}{\sqrt{n}}$ ~\cite{Cheu:2019} & \multirow{3}{*}{$\frac{e^{\eps_0/2}}{ \sqrt{n}}$ [Thm.~\ref{higheps0}]}\\
     \hhline{--~}
     Non-adaptive shuffle & $\frac{e^{\eps_0}}{\sqrt{n}}$ ~\cite{Balle:2019}& \\
     \hhline{--~}
     \multirow{2}{*}{Adaptive shuffle} & $\frac{e^{3\eps_0}}{\sqrt{n}}$~\cite{ErlingssonFMRTT19}& \\
     & $\frac{e^{2.5\eps_0}}{\sqrt{n}}$~\cite{Balle:2020}  &\\
     \hline
     Adaptive shuffle for $(\eps_0,\delta_0)$-DP& $\frac{e^{20\eps_0}}{\sqrt{n}}$~\cite{Balle:2020} & $\left(\frac{e^{\eps_0/2}}{\sqrt{n}}, \delta + n\delta_0\right)$ \;[Thm.~\ref{approxDP}]\\
     \hline
     $k$-Randomized response & $\max\left\{\frac{e^{\eps_0}+k}{n},\sqrt{\frac{e^{\eps_0}+k}{n}}\right\}$~\cite{Balle:2019}& $\frac{e^{\eps_0}}{\sqrt{n(e^{\eps_0}+k)}}$ %$\min\left\{\frac{e^{\eps_0/2}}{\sqrt{n}}, \frac{e^{\eps_0}}{\sqrt{nk}}\right\}$
     \;[Cor.~\ref{kRRthm}]\\
     \hline
    \end{tabular}
\vspace{0.1in}
    \caption{Previous and our results for privacy amplification by shuffling in the $\eps_0>1$ regime. We suppress constant factors, as well as the $\sqrt{\ln(1/\delta)}$ factor that is common to all results.}
    \label{tab:results}
\end{table}

\paragraph{Follow-up work:}
A natural direction for future work is to analyze privacy amplification by shuffling in terms of other notions of differential privacy, such as R\'enyi differential privacy (RDP) \citep{mironov2017renyi}. RDP guarantees are particularly useful in multi-step algorithms where the composition properties of differential privacy need to be used. A bound on RDP guarantees is implied by the proof technique in \citep{ErlingssonFMRTT19} but, just as in the case of approximate DP, this bound is asymptotically suboptimal when $\eps_0 > 1$. Specifically, for sufficiently small order $\alpha$, the $\alpha$-RDP divergence is $O(\alpha e^{6\eps_0})$ in this regime.

Subsequent work of \citet{Girgis:2021}, (implicitly) relies on the idea proposed in this work: it views each data point as creating a clone of the first element of the first dataset (in the given pair of datasets) and thereby reduces the problem to analysis of divergence between a pair of datasets that have just two different elements. The divergence is then analyzed using both analytic and numerical tools. Their resulting bound on the order $\alpha$ RDP is $O(\alpha e^{2\eps_0})$ (when $\eps_0 > 1$). This is a significant improvement on the results in \citep{ErlingssonFMRTT19} but is still asymptotically suboptimal.

In Section~\ref{renyiDPexperiments} we evaluate numerically the R\'enyi divergence between the pair of distributions given in our reduction and demonstrate that our approach leads to significantly better bounds
\footnote{The numerical evaluation has not been included in the earlier version of this work and thus not available to the authors of \citep{Girgis:2021}.}. We leave the computation of a closed-form bound on the R\'enyi divergence between the pair of distributions over three values that emerge from our reduction for future work.

Another subsequent work \citep{koskela2021tight} motivated by the problem of obtaining better guarantees for the composition of multiple shuffled protocols relies on  a different approach to composition referred to as Fourier accountant \citep{koskela2020computing}. The approach requires computation of the Fast Fourier Transform of the entire privacy loss variable and \citet{koskela2021tight} use the pair of distributions given in our work as the input to their analysis.

\section{Background and Preliminaries}
Differential privacy (DP) is a measure of stability of a randomized algorithm. It bounds the change in the distribution on the outputs when one of the inputs is replaced with an arbitrary other element. The most common way to measure the change in the output distribution is $(\eps, \delta)$-indistinguishability. Two random variables $P$ and $Q$ over some probability space are $(\eps, \delta)$-\emph{indistinguishable} if for all events $E$ over that probability space, \[e^{-\eps}(\Pr(Q\in E)-\delta)\le \Pr(P\in E)\le e^{\eps}\Pr(Q\in E)+\delta.\] The following {\em hockey-stick} divergence can be used to characterize $(\eps, \delta)$-indistinguishability: \[\dalpha{e^{\eps}}(P\|Q) = \int \max\{0, P(x)-e^{\eps} Q(x)\} dx,\] where we use the notation $P$ and $Q$ to refer to both the random variables and their probability density functions. So $P$ and $Q$ are $(\eps, \delta)$-indistinguishable if $\max\{\dalpha{e^{\eps}}(P\|Q), \dalpha{e^{\eps}}(Q\|P)\}\le\delta$. We will rely on several standard properties of the hockey-stick divergence such as convexity and preservation under post-processing \citep{DR14-book}.

We will consider several models of differentially private data analysis.
In the central model, introduced in \cite{Dwork:2006}, the data of the individuals is held by the curator. The curator is then trusted to perform data analysis whose output does not disclose too much about any particular individual's data. While this model requires a higher level of trust than the local model, it is possible to design significantly more accurate algorithms.
We say that two databases are neighboring if they differ on the data of a single individual.

\begin{definition}[Differential privacy in the central model]
An algorithm $\mathcal{A}:\mathcal{D}^n\to\output$ is $(\eps, \delta)$-\emph{differentially private} if for all neighboring databases $X$ and $X'$, $\mathcal{A}(X)$ and $\mathcal{A}(X')$ are $(\eps, \delta)$-indistinguishable.
\end{definition}

In the local model, formally introduced in \cite{Kasiviswanathan:2008}, each individual (or client) randomizes their data before sending it to data curator (or server). This means that individuals are not required to trust the curator. Due to this minimal trust model, most current industrial deployments of differential privacy rely on local differential privacy \cite{Erlingsson:2014, Apple2017, Bolin:2017, ErlingssonFMRSTT20}. More formally, in the general model of local DP clients holding their data can communicate with the server in an arbitrary order with multiple rounds of interaction. The protocol is said to satisfy local $(\eps,\delta)$-differential privacy if the transcripts of the protocol on any pair of neighboring datasets are $(\eps, \delta)$-indistinguishable. We will not require this general definition as we will only be considering protocols in which each client receives at most one message from the server and sends a single message to the server (which may depend on the message from the server). In this setting, the condition reduces to requiring that the algorithm the client uses to respond to the server satisfies differential privacy with respect to the input of the client, such an algorithm is often referred to as a {\em local randomizer}.

\begin{definition}[Local randomizer]\label{localrandomizer}
An algorithm $\lr\colon \D\to \cS$ is $(\eps, \delta)$-DP \emph{local randomizer}  if
for all pairs $x,x'\in \D$, $\lr(x)$ and $\lr(x')$ are $(\eps, \delta)$-indistinguishable.
\end{definition}

Let $\mathcal{D}$ denote the domain. A single pass $\eps_0$-DP local protocol
can be equivalently described as a sequence of algorithms $\Aldp[i]:\out[1]\times\cdots\times\out[i-1]\times\mathcal{D}\to\out[i]$ for $i\in[n]$ (where $\out[i]$ is the range space of $\Aldp[i]$) where the $i$-th client returns $z_i=\Aldp[i](z_{1:i-1}, x_i)$, and $\Aldp[i](z_{1:i-1}, \cdot)$ is an $(\eps,\delta)$-DP local randomizer for all values of auxiliary inputs $z_{1:i-1}\in\out[1]\times\cdots\times\out[i-1]$. The dependence on $z_{1:i-1}$ models the fact that the server can communicate to the client $i$ any information based on the messages from clients $1,\ldots, i-1$.

The shuffle model of privacy is a distributed model of computation in which, as in the local model, clients hold their data and a server communicates with the clients to perform data analysis. In addition, the model includes a shuffler (also referred to as a mixnet). The shuffler collects all the reports sent by clients and anonymizes them by applying a random permutation to all the reports. The shuffled reports are then released to the server. Note that the reports are typically encrypted in a way that they can be decrypted by the server but not by the shuffler (a more detailed discussion of the trust assumptions can be found in \cite{ErlingssonFMRSTT20}). A server can also communicate back to the clients and the protocol may consist of multiple rounds of communication via a shuffler. Such a protocol is said to satisfy $(\eps,\delta)$-differential privacy in the shuffle model if the outputs of the shuffler on any pair of neighboring datasets are $(\eps, \delta)$-indistinguishable. We remark, that such protocols do not necessarily satisfy local differential privacy.

The anonymization of local reports to improve differential privacy was proposed in \cite{Bittau17}, who designed and implemented a principled systems architecture for shuffling. The intuition was formalized in  \cite{ErlingssonFMRTT19, Cheu:2019}. \citet{ErlingssonFMRTT19} showed that for arbitrary $\eps_0$-DP local randomizers random shuffling of reports amplifies the privacy guarantees. \citet{Cheu:2019} formally defined the shuffle model of computation and analyzed the privacy guarantees of the binary randomized response in this model. It is important to note that, as in \citep{ErlingssonFMRTT19}, we analyze privacy amplification that results from shuffling of the data before applying local randomizers. In contrast, in the shuffle model the shuffler is applied to reports from the clients. To apply our analysis to the shuffle model, it suffices to observe that for a fixed local randomizer, shuffling of the randomized responses is distributed in the same way as randomized reports on shuffled data. In particular, it enjoys the privacy guarantees that we establish.

In all models, if $\delta=0$ then we will refer to an algorithm as $\eps$-differentially private. Further, we will occasionally refer to $\delta=0$ as pure differentially private and $\delta>0$ as approximate differential privacy.

Subsampling is a well known technique for amplifying privacy guarantees \citep{Kasiviswanathan:2008}. That is, suppose that $\A$ is an $(\eps,\delta)$-DP algorithm over datasets of size $m$ and define an algorithm $\A'$ over a dataset $X$ of size $n > m$ by picking a random and uniform subset of $m$ elements from $X$ and running $\A$ on the resulting dataset. Privacy amplification by subsampling states that $\A'$ is $\left(\log(1+\frac{m}{n}(e^{\eps}-1)),  \frac{m}{n} \delta\right)$-DP \citep{Kasiviswanathan:2008,Balle:2018}. The following lemma is a slightly more general and abstract version of this result.
\begin{lemma}[Advanced joint convexity, \cite{Balle:2018}]\label{advancedjointconvexity}
Let $P$ and $Q$ be distributions satisfying $P=(1-q)P_0+qP_1$ and $Q=(1-q)P_0+qQ_1$ for some $q\in[0,1]$ and distributions $P_0, P_1$ and $Q_1$. Given $\alpha\ge 1$, let $\alpha'=1+q(\alpha-1)$ and $\theta=\alpha'/\alpha$. Then the following holds \[\dalpha{\alpha'}(P\|Q)=q\dalpha{\alpha}(P_1\|(1-\theta)P_0+\theta Q_1).\] In particular, for any $\eps>0$, if $\eps'=\log(1+q(e^{\eps}-1))$
then
\[\dalpha{e^{\eps'}}(P\|Q)\le q\max\{\dalpha{e^{\eps}}(P_1\|P_0), \dalpha{e^{\eps}}(P_1\|Q_1)\}.\]
\end{lemma}

\section{Improved Guarantees for Privacy Amplification by Shuffling}

In this section we present our main theorem, an improved upper bound for privacy amplification via shuffling for adaptive $\eps_0$-DP randomizers. Theorem~\ref{higheps0} provides a clean statement of our improved privacy amplification guarantee, improving on previous bounds in the high $\eps_0$ regime. The main technical statement is contained in Theorem~\ref{higheps0}.

\begin{theorem}\label{higheps0}
 For any domain $\mathcal{D}$, let $\Aldp[i]:\out[1]\times\cdots\times\out[i-1]\times\mathcal{D}\to\out[i]$ for $i\in[n]$ (where $\out[i]$ is the range space of $\Aldp[i]$) be a sequence of algorithms such that $\Aldp[i](z_{1:i-1}, \cdot)$ is an $\eps_0$-DP local randomizer for all values of auxiliary inputs $z_{1:i-1}\in\out[1]\times\cdots\times\out[i-1]$.
 Let $\shuffler:\mathcal{D}^n\to\out[1]\times\cdots\times \out[n]$ be the algorithm that given a dataset $x_{1:n}\in\mathcal{D}^n$, samples a uniform random permutation $\pi$ over $[n]$, then sequentially computes $z_i=\Aldp[i](z_{1:i-1}, x_{\pi(i)})$ for $i\in[n]$ and outputs $z_{1:n}$.
 Then for any $\delta\in[0,1]$ such that $\eps_0\le\log(\frac{n}{16\log(2/\delta)})$, $\shuffler$ is $(\eps, \delta)$-DP,  where
\begin{equation}\label{epsbound}
\eps\le \log\left(1+\frac{e^{\eps_0}-1}{e^{\eps_0}+1}\left(\frac{8\sqrt{e^{\eps_0}\log(4/\delta)}}{\sqrt{n}}+\frac{8e^{\eps_0}}{n}\right)\right)
\end{equation}
\end{theorem}

Note that when $\eps_0> 1$, $\eps=O\left(\frac{\sqrt{e^{\eps_0}\log(1/\delta)}}{\sqrt{n}}\right)$ and when $\eps_0 \leq 1$, $\eps=O\left(\eps_0\frac{\sqrt{\log(1/\delta)}}{\sqrt{n}}\right)$.

A natural question is whether the given bound also holds when $\eps_0 > \log(\frac{n}{16\log(2/\delta)})$ or, equivalently, $\log(2/\delta)\le \frac{ne^{\epsilon_0}}{16}$. In Appendix~\ref{app:verify-tails} we give numerical evidence that there is no privacy amplification when $\log(1/\delta)=\Omega(ne^{-\epsilon_0})$.

Our proof relies on a more general result showing that in order to analyze the privacy amplification by shuffling it suffices to upper bound divergence between a specific pair of distributions.
\begin{theorem}\label{shuffletobinoms}
For a domain $\mathcal{D}$, let $\Aldp[i]:\out[1]\times\cdots\times\out[i-1]\times\mathcal{D}\to\out[i]$ for $i\in[n]$ (where $\out[i]$ is the range space of $\Aldp[i]$) be a sequence of algorithms such that $\Aldp[i](z_{1:i-1}, \cdot)$ is an $\eps_0$-DP local randomizer for all values of auxiliary inputs $z_{1:i-1}\in\out[1]\times\cdots\times\out[i-1]$. Let $\shuffler:\mathcal{D}^n\to \out[1]\times\cdots\times \out[n]$ be the algorithm that given a dataset $x_{1:n}\in\mathcal{D}^n$, samples a permutation $\pi$ uniformly at random, then sequentially computes $z_i=\Aldp[i](z_{1:i-1}, x_{\pi(i)})$ for $i\in[n]$ and outputs $z_{1:n}$. Let $X_0$ and $X_1$ be two arbitrary neighboring datasets in $\D^n$.  Let $C\sim \bin(n-1, e^{-\eps_0})$, $A\sim\bin(C, 1/2)$ and $\Delta \sim \Ber\left(\frac{e^{\eps_0}}{e^{\eps_0}+1}\right)$, where $\Ber(p)$ denotes a Bernoulli random variable with bias $p$. Then there exists a randomized postprocessing algorithm $\Phi$ such that $\shuffler(X_0)$ is distributed identically to $\Phi(A+\Delta, C-A+1-\Delta)$ and $\shuffler(X_1)$ is distributed identically to  $\Phi(A+1-\Delta,C-A+\Delta)$.
\end{theorem}

In addition to giving our closed-form upper bound, Theorem~\ref{shuffletobinoms}, provides an efficient method for numerically computing a tighter upper bound. We describe this numerical approach in Section~\ref{experiment}.

Now let us turn to the proof of Theorems~\ref{higheps0} and~\ref{shuffletobinoms}.
As mentioned earlier, our proof involves reducing the problem of analyzing the privacy guarantee of an adaptive series of local randomizers applied to shuffled data to analyzing the privacy guarantee of a simple non-adaptive local algorithm with three outputs. Suppose that $X_0$ and $X_1$ are neighbouring databases that differ on the first datapoint, $x_1^0\neq x_1^1$. The reduction is outlined in Figure~\ref{graphoutline}, which shows the behavior of the simpler algorithm on data set $X_0$.
The key observation is that for any $\eps_0$-DP local randomizer $\lr$ and data point $x$,
$\lr(x)$ can be seen as sampling from the same distribution as $\lr(x_1^0)$ with probability at least $e^{-\eps_0}/2$ and sampling from the same distribution as $\lr(x_1^1)$ with probability at least $e^{-\eps_0}/2$.
That is, with probability $e^{-\eps_0}$ each data point can create a clone of the output of $\lr(x_1^0)$ or a clone of $\lr(x_1^1)$ with equal probability. Thus $n-1$ data elements effectively produce a random number of clones of both $x_1^0$ and $x_1^1$. These clones make distinguishing whether the original data set contains $x_1^0$ or $x_1^1$ as its first element much harder.

The proof of Theorem~\ref{higheps0} will require an additional step where we observe that if $\lr$ is $\eps_0$-DP then $\lr(x_1^0)$ and $\lr(x_1^1)$ are similar, and thus privacy is further amplified. Let us start by proving Lemma~\ref{binomstomixtures}, which captures just the privacy amplification due to the generation of clones.
The following lemma analyses the divergence between the distributions on the number of clones that result from this reduction.

For notational brevity, we will frequently use the same symbol to refer to both a random variable and it's probability density function. In particular, if $X$ and $Y$ are random variables and $\alpha\in[0,1]$ then $Z=\alpha X+(1-\alpha) Y$ denotes the random variable that samples from $X$ with probability $\alpha$ and $Y$ with probability $1-\alpha$. Consequently, for a randomized algorithm $\mathcal{A}$ and input $x$, we treat the output $\mathcal{A}(x)$ as a random variable and also use $\mathcal{A}(x)$ to refer to the probability density function of this random variable.

Lemma~\ref{binomstomixtures} describes our reduction from the original problem to analysis of shuffling for a simple non-adaptive protocol. Specifically, it shows that if each local randomizer $\lr$ we apply can be decomposed into a mixture of $\lr(x_1^0)$, $\lr(x_1^1)$ and some ``left-over'' distribution $\lo(x)$ then we can apply our reduction. We will then conclude the proof of Theorem~\ref{higheps0} by observing that that every differentially private randomizer has this property.

\begin{lemma}\label{binomstomixtures}
For a domain $\mathcal{D}$, let $\Aldp[i]:\out[1]\times\cdots\times\out[i-1]\times\mathcal{D}\to\out[i]$ for $i\in[n]$ (where $\out[i]$ is the range space of $\Aldp[i]$) be a sequence of algorithms such that $\Aldp[i](z_{1:i-1}, \cdot)$ is an $\eps_0$-DP local randomizer for all values of auxiliary inputs $z_{1:i-1}\in\out[1]\times\cdots\times\out[i-1]$. Let $\shuffler:\mathcal{D}^n\to \out[1]\times\cdots\times \out[n]$ be the algorithm that given a dataset $x_{1:n}\in\mathcal{D}^n$, samples a permutation $\pi$ uniformly at random, then sequentially computes $z_i=\Aldp[i](z_{1:i-1}, x_{\pi(i)})$ for $i\in[n]$ and outputs $z_{1:n}$. Let $X_0=(x^0_1,x_2,\ldots,x_n)$ and $X_1=(x^1_1,x_2,\ldots,x_n)$ be two neighboring datasets such that for all $j\neq 1$, $x_j\notin\{x_1^0, x_1^1\}$. Suppose that there exists a positive value $p\in(0,1]$ such that for all $i\in[n]$, $x\in \D\setminus\{x_1^0,x_1^1\}$ and  $z_{1:i-1} \in  \out[1]\times\cdots\times \out[i-1]$, there exists a distribution $\lo^{(i)}(z_{1:i-1}, x)$ such that
\begin{equation}\label{mix}
    \Aldp[i](z_{1:i-1}, x) = \frac{p}{2}\Aldp[i](z_{1:i-1}, x^0_1)+\frac{p}{2}\Aldp[i](z_{1:i-1}, x^1_1)+(1-p)\lo^{(i)}(z_{1:i-1}, x).
    \end{equation}
Then there exists a randomized postprocessing algorithm $f$ such that $\shuffler(X_0)$ is distributed identically to $f(A+1, C-A)$ and $\shuffler(X_1)$ is distributed identically to  $f(A,C-A+1)$, where $C\sim \bin(n-1, p)$, $A\sim\bin(C, 1/2)$.
\end{lemma}
\begin{proof}
The proof of the above lemma relies on a decomposition of the algorithm that shuffles the data and then applies the local randomizers, to an algorithm in which each client first reports which component of the mixture it will sample from, then applies shuffling to these reports and finally applies a post-processing step in which randomizers are applied according to the shuffled mixture component indices.

\noindent Formally, define a random variable $Y$ as follows \[Y = \begin{cases} 0 & \text{w.p.   } p/2\\ 1 & \text{w.p.   } p/2 \\ 2 & \text{w.p.   } 1-p \end{cases}.\] Given a dataset $X_b$ for $b \in \zo$ we generate $n$ samples from $\zot$ in the following way. Client number one (holding the first element of the dataset) reports $b$. Clients $2,\ldots,n$ each report an independent sample from $Y$. We then shuffle the reports randomly. Let $\rho_b$ denote the resulting distribution over $\zot^n$.

We claim that there exists a post-processing function $f$ (that depends on $x^0_1,x^1_1,x_2,\ldots, x_n$) such that for $y$ sampled from $\rho_b$, $f(y)$ is distributed identically to $\shuffler(X_b)$. To see this, consider the following process that takes a $b\in \zo$ as an input and defines a jointly distributed pair of random variables $(z,y)$ over $\cS \times \zot^n$, where $\cS = \out[1]\times\cdots\times \out[n]$ is the output space of $\shuffler$. Let $\pi$ be a randomly and uniformly chosen permutation of $[n]$. For every $i \in [n]$, if $\pi(i)\neq 1$ then we first sample $y_i$ according to $Y$, and if $i=1$, we set $y_i=b$. We then use the $y_i$'s to generate $z_i$'s. Specifically,
 $$z_i \sim  \begin{cases} \Aldp[i](z_{1:i-1}, x^0_1) & \text{if } y_i=0; \\
                                    \Aldp[i](z_{1:i-1}, x^1_1) & \text{if } y_i=1;  \\
                                    \lo^{(i)}(z_{1:i-1}, x_{\pi(i)}) & \text{if } y_i=2. \end{cases}$$
By our assumption, this produces a sample $z_i$ from $\Aldp[i](z_{1:i-1}, x_{\pi(i)})$. It is easy to see that the resulting random variable $(z,y)$ has the property that for input $b \in \zo$ its marginal distribution over $\cS$ is the same as $\shuffler(X_b)$ and marginal distribution over $\zot^n$ is $\rho_b$.

We can now show that for every fixed value $v \in \zot^n$ and $b\in \zo$, one can generate a sample from the distribution of $z$ conditioned on $y = v$ without knowing $b$ from which the pair $(z,y)$ was generated. First observe that the above construction of $(y, z, \pi)$ has the property that for any permutation $\sigma$, conditioned on $(y=v, \pi=\sigma)$, the random variable $z$ does not depend on $b$. A natural approach to sample from $z|_{y=v}$ would be to sample a permutation $\sigma$ from the distribution of $\pi$ conditioned on $y=v$, and then sample $z|_{(y,\pi)=(v, \sigma)}$. The conditional distribution of $\pi|_{y=v}$ however is not independent of $b$. Let $T = \pi(\{i: y_i = 2\})$ be the indices corresponding to $v_i=2$. First observe that $T|_{y=v}$ is independent of $b$, since this distribution is uniform over subsets of $\{2,\ldots, n\}$ of the appropriate size. Finally, note that the sampling of $z$ given $y$ only needs $T$. Thus we can sample from $z|_{(y,T)=(v,J)}$ without knowing $b$. This conditional sampling is exactly the post-processing step that we claimed. We include a formal description of the post-processing step in Algorithm~\ref{postprocessing}.

We now analyze the divergence between $\rho_0$ and $\rho_1$. First, the shuffling step implies that $\rho_0$ and $\rho_1$ are symmetric. That is, the probability of any sequence $y \in \zot^n$ depends only on the total number of 0s and 1s in the sequence. This implies that  the divergence between $\rho_0$ and $\rho_1$ is equal to the divergence between the distribution of the counts of 0's and 1's. Let $(c_0,c_1)=(\sum_{i\in[n]} \indicator_{y_i=0}, \sum_{i\in[n]} \indicator_{y_i=1}$).
It now suffices to observe that, if $C\sim\bin(n-1, p)$ and $A\sim\bin(C,1/2)$, then the counts for $y\sim \rho_0$ are distributed as $(1+A,C-A)$ and the count for $y\sim \rho_1$ is distributed as $(A,C-A+1)$.

\begin{algorithm}
  \textbf{Input:} $x^0_1,x^1_1,x_2,\ldots, x_n$; $y \in \zot^n$\\
  $J := \emptyset$\\
  \For{$i=1,\ldots,n$}{
  \If{$y_i =2$}{
  Let $j_i$ be a randomly and uniformly chosen element of $[2:n]\setminus J$\\
  $J := J \cup \{j_i\}$
  }
  Sample $z_i$ from $ \begin{cases} \Aldp[i](z_{1:i-1}, x^0_1) & \text{if } y_i=0; \\
                                    \Aldp[i](z_{1:i-1}, x^1_1) & \text{if } y_i=1;  \\
                                    \lo^{(i)}(z_{1:i-1}, x_{j_i}) & \text{if } y_i=2. \end{cases}$\\
  }
  \textbf{return} $z_1, \ldots, z_n$
  \caption{Post-processing function, $f$} \label{postprocessing}
\end{algorithm}

\end{proof}

\begin{figure}
\includegraphics[scale=0.4]{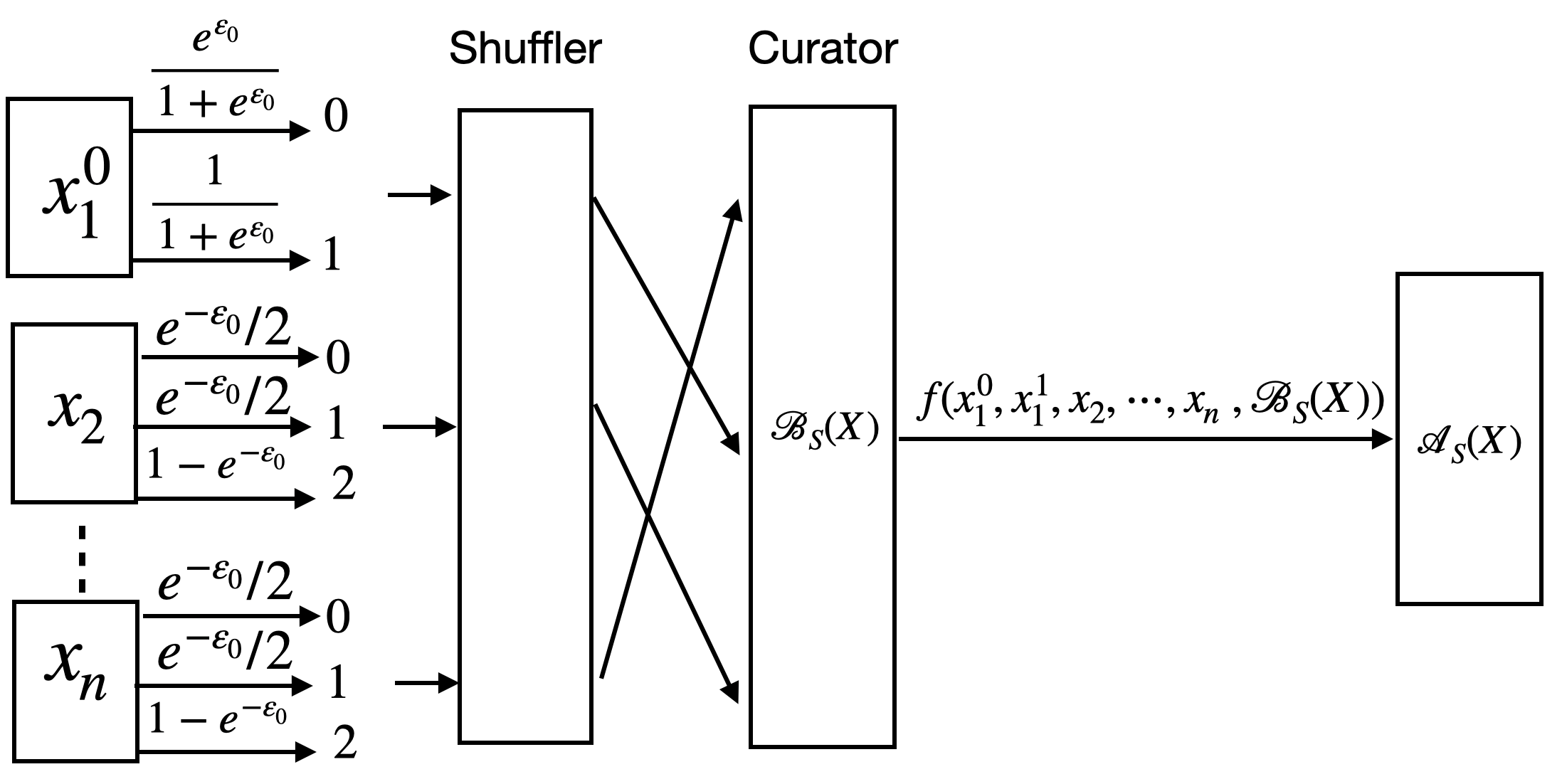}
\caption{Graphical representation of the reduction in the proof of Theorem~\ref{higheps0}}\label{graphoutline}
\end{figure}

We can now complete the proof of Theorem~\ref{higheps0}. To do this we will use privacy amplification provided by the clones (Lemma \ref{binomstomixtures}) together with the observation that for any $\eps_0$-DP randomizer $\lr$, $\lr(x^0_1)$ and $\lr(x^1_1)$ are $\eps_0$-indistinguishable. To fit this into our analysis we will use the fact that $\lr(x^0_1)$ and $\lr(x^1_1)$ can be viewed as post-processing of a binary randomized response with parameter $\eps_0$  \citep{Kairouz:2015} (see also \citep{murtagh2016complexity}).
\begin{lemma}[\citep{Kairouz:2015}]\label{lem:kov}
Let $\lr\colon \D \to \cS$ be an $\eps_0$-DP local randomizer and $x_0,x_1 \in \D$. Then there exists a randomized algorithm $\Q \colon \zo \to \cS$ such that $\lr(x_0) =  \frac{e^{\eps_0}}{e^{\eps_0}+1} \Q(0) + \frac{1}{e^{\eps_0}+1} \Q(1)$ and $\lr(x_1) =  \frac{1}{e^{\eps_0}+1} \Q(0) + \frac{e^{\eps_0}}{e^{\eps_0}+1} \Q(1)$.
\end{lemma}
Binary randomized response can be seen as an algorithm that outputs its input with probability $\frac{e^{\eps_0}-1}{e^{\eps_0}+1}$ and outputs an unbiased coin flip with probability $\frac{2}{e^{\eps_0}+1}$. Thus it can be seen as providing privacy amplification by subsampling with probability $\frac{e^{\eps_0}-1}{e^{\eps_0}+1}$. Note that for $\eps_0 <1$ this probability is  $O(\eps_0)$. This will ensure that our bound is accurate in the regime $\eps_0 <1$.

\begin{proof}[Proof of Theorem~\ref{shuffletobinoms}]
Let $X_0$ and $X_1$ be neighboring datasets. We can assume without loss of generality that $X_0=(x^0_1,x_2,\ldots,x_n)$ and $X_1=(x^1_1,x_2,\ldots,x_n)$ for $x_1^0\neq x_1^1$ (since the shuffler applies a random permutation). Further, we can assume without loss of generality that for all $j\neq 1$, $x_j\notin\{x_1^0, x_1^1\}$. This true since we can add elements ${x'}_1^0$ and ${x'}_1^1$ to the domain $\D$ and define $\Aldp[i]$ on them in the same way as on  $x_1^0$ and $x_1^1$, respectively. With this definition, for $b \in \zo$, the output distribution of $\shuffler({x'}_1^b,x_2,\ldots,x_n)$ is identical to that of $\shuffler(X_b)$.

By Lemma \ref{lem:kov} we have that for every $i$ and value of auxiliary input $z_{1:i-1}$
there exists a randomized algorithm $\Q^{(i)}(z_{1:i-1}, \cdot) \colon \{x_1^0,x_1^1\} \to \out[i]$ such that
\equ{
\Aldp[i](z_{1:i-1}, x_1^0)=\frac{e^{\eps_0}}{e^{\eps_0}+1}\Q^{(i)}(z_{1:i-1},x_1^0)+\frac{1}{e^{\eps_0}+1}\Q^{(i)}(z_{1:i-1},x_1^1) \label{eq:decomp0}} and
\equ{\Aldp[i](z_{1:i-1}, x_1^1)=\frac{1}{e^{\eps_0}+1}\Q^{(i)}(z_{1:i-1},x_1^0)+\frac{e^{\eps_0}}{e^{\eps_0}+1}\Q^{(i)}(z_{1:i-1},x_1^1) \label{eq:decomp1}} where again for notational brevity we use the same notation for a random variable and its probability density function.

Next we observe that the definition of $\eps_0$-DP directly implies that for an $\eps_0$-DP randomizer $\lr \colon \D \to \cS$ and any input $x_0 \in \D$, there exists a randomized algorithm $\lo \colon \D \to \cS$ such that $\lr(x)$ can be decomposed as $$\lr(x) = \frac{1}{e^{\eps_0}} \lr(x_0) + \left(1- \frac{1}{e^{\eps_0}} \right) \lo(x) .$$ Applying this argument twice we obtain that, for any pair of inputs $x_0,x_1 \in \D$ there exists a decomposition $$\lr(x) = \frac{1}{2 e^{\eps_0}} \lr(x_0) + \frac{1}{2 e^{\eps_0}} \lr(x_1) + \left(1- \frac{1}{e^{\eps_0}} \right)\lo'(x) .$$
Applying this observation to $\Aldp[i](z_{1:i-1}, \cdot)$, for all $i\in [n]$ and values of $z_{1:i-1}$, we have that there exists a randomized algorithm $\lo^{(i)}$ such that for $p = e^{-\eps_0}$ we have:
\[    \Aldp[i](z_{1:i-1}, x) = \frac{p}{2}\Aldp[i](z_{1:i-1}, x^0_1)+\frac{p}{2}\Aldp[i](z_{1:i-1}, x^1_1)+(1-p)\lo^{(i)}(z_{1:i-1}, x).
\]

Note that
\[\Aldp[i](z_{1:i-1}, x_1^0) + \Aldp[i](z_{1:i-1}, x_1^1)= \Q^{(i)}(z_{1:i-1},x_1^0)+\Q^{(i)}(z_{1:i-1},x_1^1) \] and therefore we also obtain that
\begin{equation}\label{mixture}\Aldp[i](z_{1:i-1}, x) = \frac{p}{2}\Q^{(i)}(z_{1:i-1},x_1^0)+\frac{p}{2}\Q^{(i)}(z_{1:i-1},x_1^1)+(1-p)\lo^{(i)}(z_{1:i-1}, x),
\end{equation}

Next we consider the algorithm $\A_\Q$ which is defined in the same way as $\shuffler$, except $\Aldp[i](z_{1:i-1},x_1^b)$ is replaced with $\Q^{(i)}(z_{1:i-1},x_1^b)$. Formally, we define a randomizer $\lr_\Q^{(i)}$ as follows:
For all $x\in \D$, $i\in[n]$ and values of $z_{1:i-1}$ we let
\[\lr_\Q^{(i)}(z_{1:i-1}, x)=
\begin{cases} \Q^{(i)}(z_{1:i-1},x_1^0) & \text{if } x=x_1^0; \\
                                    \Q^{(i)}(z_{1:i-1},x_1^1) & \text{if } x=x_1^1;  \\
                                    \Aldp[i](z_{1:i-1}, x) & \text{otherwise.} \end{cases}\]
Let $\A_\Q$ be defined in the same way as $\shuffler$, except $\Aldp[i]$ is replaced with $\lr_\Q^{(i)}$.

Equations~\eqref{eq:decomp0} and~\eqref{eq:decomp1}, allow us to decompose $\shuffler(X_0)$ and $\shuffler(X_1)$ into the mixture of two components as follows:
\begin{equation}\label{mixxymix}\shuffler(X_0) = \frac{e^{\eps_0}}{e^{\eps_0}+1}\A_\Q(X_0)+\frac{1}{e^{\eps_0}+1}\A_\Q(X_1)\;\;\text{and}\;\;\shuffler(X_1) = \frac{1}{e^{\eps_0}+1}\A_\Q(X_0)+\frac{e^{\eps_0}}{e^{\eps_0}+1}\A_\Q(X_1).
\end{equation}

Note that by eq.~\eqref{mixture}, for all $x\notin\{x_1^0,x_1^1\}$,
\[\lr_\Q^{(i)}(z_{1:i-1}, x) = \frac{p}{2}\lr_\Q^{(i)}(z_{1:i-1}, x_1^0)+\frac{p}{2}\lr_\Q^{(i)}(z_{1:i-1}, x_1^1)+\left(1-p\right)\lo^{(i)}(z_{1:i-1},x).\] Therefore, by Lemma~\ref{binomstomixtures}, there exists a postprocessing function $f$ such that $\A_\Q(X_0)$ is distributed identically to $f(A+1, C-A)$ and $\A_\Q(X_1)$ is distributed identically to  $f(A,C-A+1)$, where $C\sim \bin(n-1, e^{-\eps_0})$, $A\sim\bin(C, 1/2)$. Now, the decomposition in eq.~\eqref{mixture} implies that $\shuffler(X_0)$ is distributed identically to $f(A+\Delta, C-A+1-\Delta)$ and $\shuffler(X_1)$ is distributed identically to  $f(A+1-\Delta,C-A+\Delta)$, $\Delta \sim \Ber\left(\frac{e^{\eps_0}}{e^{\eps_0}+1}\right)$.

\end{proof}

To complete the proof of Theorem~\ref{higheps0}, all we need is the following lemma on the hockey-stick divergence between the distributions resulting from Theorem~\ref{shuffletobinoms} and the post-processing inequality for the hockey-stick divergence.
\begin{restatable}{lemma}{rbinomials}\label{binomials} Let  $n\in\mathbb{N}$, $\delta\in(0,1]$ and $\eps_0\le\log(\frac{n}{16\log(2/\delta)})$.
Consider the process where we sample $C\sim \bin(n-1, e^{-\eps_0})$, $A\sim\bin(C, 1/2)$ and $\Delta \sim \Ber\left(\frac{e^{\eps_0}}{e^{\eps_0}+1}\right)$. Let $P=(A+\Delta,C-A+1-\Delta)$ and $Q=(A+1-\Delta,C-A+\Delta)$, then
$P$ and $Q$ are $(\eps,\delta)$-indistinguishable for
\begin{equation}\label{amplificationboundeq}
\eps= \log\left(1+\frac{e^{\eps_0}-1}{e^{\eps_0}+1}\left(\frac{8\sqrt{e^{\eps_0}\log(4/\delta)}}{\sqrt{n}}+\frac{8e^{\eps_0}}{n}\right)\right). \end{equation}
\end{restatable}
We defer the proof to Appendix~\ref{binomialsproof}.

\subsection{Generalization to Approximate Differential Privacy}

In order to analyse the approximate DP case, we can use a reduction from $(\eps_0,\delta_0)$-DP local randomizers to $\eps_0'$-DP local randomizers. We use a black box argument showing that for every $(\eps_0,\delta_0)$-DP local randomizer $\A$ there exists an $\eps_0'$-DP local randomizer $\A$ that is $\delta_0'$ close in total variation distance to $\A$. The output of $n$ $(\eps_0,\delta_0)$-DP randomizers on shuffled data is then $n\delta_0'$ close is total variation distance to the output of $n$ $\eps_0'$-DP randomizers. A fairly involved reduction of this type is given in \citep{Cheu:2019} (building on the results from \citep{bun2019heavy}) and achieves $\eps_0' = 8\eps_0$ and $\delta_0' = \tilde O (\delta e^{\eps_0}/(1-e^{-\eps_0}))$. Applying this reduction to their shuffling amplification bound for pure DP, \citet{Balle:2020} obtain a resulting $\eps$ of order $\Theta\left(\frac{e^{30\eps_0}}{n}+\frac{e^{20\eps_0}\sqrt{\log(1/\delta)}}{\sqrt{n}}\right)$ in the high $\eps_0$ regime.

Our bound is based on a simpler and stronger reduction of this type that was also independently shown in \cite{cheu2020limits}. Its tightest form is best stated in the add/delete variant of differential privacy which was defined for local randomizers in \citep{ErlingssonFMRSTT20}.
\begin{definition}[Local Deletion Differential Privacy]
An algorithm $\lr\colon \mathcal{D}\to \cS$ is a {\em deletion $(\eps, \delta)$-DP local randomizer} if
there exists a reference distribution $\reference$ such that for all data points $x\in\mathcal{D}$, $\lr(x)$ and $\reference$ are $(\eps, \delta)$-indistinguishable.
\end{definition}
We will occasionally refer to a function that satisfies Definition~\ref{localrandomizer} as a \emph{replacement} $(\eps,\delta)$-DP local randomizer.
It is easy to show that a replacement $(\eps,\delta)$-DP algorithm is also a deletion $(\eps,\delta)$-DP algorithm, and that
a deletion $(\eps,\delta)$-DP algorithm is also a replacement $(2\eps,2\delta)$-DP algorithm. The following Lemma allows us to convert $(\eps,\delta)$-DP local randomizers to $\eps$-DP local randomizers that are within $\delta$ in total variation distance.

\begin{restatable}{lemma}{rapproxtopure}\label{approxtopure}
Suppose $\lr$ is a deletion $(\eps, \delta)$-DP local randomizer with reference distribution $\reference$. Then there exists a randomizer $\lr'$ that is a deletion $\eps$-DP local randomizer with reference distribution $\reference$, and for all inputs $x$, $\TV(\lr(x), \lr'(x))\le \delta$. In particular, $\lr'$ is a (replacement) $2\eps$-DP local randomizer.
\end{restatable}
The proof can be found in Appendix \ref{app:approxproofs}.
We note that when this transformation is applied to an $(\eps,\delta)$-DP local randomizer that is also deletion $(\eps/2,\delta)$-DP (such as, for example, addition of isotropic Gaussian noise to a vector from a Euclidean ball of bounded radius) then the extra factor $2$ can be avoided. To avoid incurring this factor of $2$ in our amplification by shuffling bound, we apply the core observation from the proof of Lemma \ref{approxtopure} directly in our proof of the amplification bound. This leads to the following resulting bound which we prove in Appendix \ref{app:approxproofs}.

\begin{restatable}{thm}{rapproxDP}~\label{approxDP}
For a domain $\D$, let $\Aldp[i]\colon f\times\mathcal{D}\to\out[i]$ for $i\in[n]$ (where $\out[i]$ is the range space of $\Aldp[i]$) be a sequence of algorithms such that $\Aldp[i](z_{1:i-1}, \cdot)$ is a $(\eps_0, \delta_0)$-DP local randomizer for all values of auxiliary inputs $z_{1:i-1}\in\out[1]\times\cdots\times\out[i-1]$. Let $\shuffler\colon \mathcal{D}^n\to\out[1]\times\cdots\times \out[n]$ be the algorithm that given a dataset $x_{1:n}\in\mathcal{D}^n$, samples a uniformly random permutation $\pi$, then sequentially computes $z_i=\Aldp[i](z_{1:i-1}, x_{\pi(i)})$ for $i\in[n]$ and outputs $z_{1:n}$. Then for any $\delta\in[0,1]$ such that $\eps_0\le\log(\frac{n}{16\log(2/\delta)})$, $\shuffler$ is
 $(\eps, \delta+(e^{\eps}+1)(1+e^{-\eps_0}/2)n\delta_0)$-DP, where $\eps$ is as in Equation~\eqref{epsbound}.
\end{restatable}

\new{Notice that the bound on $\eps$ in Theorem~\ref{approxDP} matches the bound in Theorem~\ref{higheps0}. A natural question is if the dependence of $\delta$ on $n$ and $\eps_0$ is optimal. We leave this question for future work.}

\section{A Tighter Analysis for Specific Randomizers}

In this section we give a more refined analysis of the amplification by shuffling that can exploit additional properties of the local randomizer and lead to improved privacy amplification results for specific randomizers. In particular, Theorem~\ref{higheps0new} will immediately imply a privacy amplification by shuffling result for $k$ randomized response ($\kRR$). To our knowledge, this is the first theoretical result that shows that privacy amplification improves as $k$ increases, corroborating empirical results from \cite{Balle:2019}. As with Theorem~\ref{higheps0}, Theorem~\ref{higheps0new} also gives us a method for empirically computing the privacy guarantee.

We will show the amplification bound is stronger if, in addition, $\Aldp[i](z_{1:i-1},x_1^0)$ and $\Aldp[i](z_{1:i-1},x_1^1)$ are close in total variation distance, specifically when the total variation distance is less than $2/(e^{\eps_0}+1)$ (which is implied by $\Aldp[i]$ being $\eps_0$-DP).

\def\qzi{\Q^{(i)}_0(z_{1:i-1})}
\def\qozi{\Q^{(i)}_1(z_{1:i-1},x_1^0)}
\def\qooi{\Q^{(i)}_1(z_{1:i-1},x_1^1)}
\newcommand{\qo}[1]{\Q^{(i)}_1(z_{1:i-1},#1)}

We first state the most general form of this amplification theorem. We will use $\mnom(n;p_1,\ldots,p_k)$ to denote the multinomial distribution over $k$-tuples with $n$ samples and probabilities of each of the outcomes being $p_1,\ldots,p_k$ (where $\sum_{i\in[k]}p_k=1)$.
\begin{restatable}{theorem}{rhighepsOnew}\label{higheps0new}
For a domain $\mathcal{D}$, let $\Aldp[i]:\out[1]\times\cdots\times\out[i-1]\times\mathcal{D}\to\out[i]$ for $i\in[n]$ (where $\out[i]$ is the range space of $\Aldp[i]$) be a sequence of algorithms such that $\Aldp[i](z_{1:i-1}, \cdot)$ is a $\eps_0$-DP local randomizer for all values of auxiliary inputs $z_{1,i-1}\in\out[1]\times\cdots\times\out[i-1]$. Let $\shuffler:\mathcal{D}^n\to\out[1]\times\cdots\times \out[n]$ be the algorithm that given a dataset $x_{1:n}\in\mathcal{D}^n$, samples a uniformly random permutation $\pi$, then sequentially computes $z_i=\Aldp[i](z_{1:i-1}, x_{\pi(i)})$ for $i\in[n]$ and outputs $z_{1:n}$. Let $X_0=(x^0_1,x_2,\ldots,x_n)$ and $X_1=(x^1_1,x_2,\ldots,x_n)$ be two neighboring datasets  such that for all $j\neq 1$, $x_j\notin\{x_1^0, x_1^1\}$. Suppose that there exist positive values $p\in(0,1/3]$ and $q\in(0,1)$ such that for all $i\in[n]$ and $z_{1:i-1} \in  \out[1]\times\cdots\times \out[i-1]$, there exists distributions $\qzi, \qozi$ and $\qooi$ such that for all $b\in \zo$,
\begin{equation}\label{TVeq}
\Aldp[i](z_{1:i-1}, x_1^b) = q \qo{x_1^b}+(1-q)\qzi
\end{equation}
and for all $x\in \D\setminus\{x_1^0,x_1^1\}$, there exists a distribution $\lo{(i)}(z_{1:i-1}, x)$ such that
\begin{equation}\label{hidingpair}
\Aldp[i](z_{1:i-1}, x) =  p\qo{x_1^0}+p\qo{x_1^1}+p\qzi+ (1-3p)\lo_b^{(i)}(z_{1:i-1}, x) .
\end{equation}
Let $(A,B,C,D)\sim\mnom(n-1;p,p,p,1-3p)$ and $\Gamma \sim \Ber(q)$. There exists a randomized postprocessing algorithm $\Phi$ such that $\shuffler(X_0)$ is distributed identically to $\Phi(A+\Gamma, B, C+1-\Gamma)$ and $\shuffler(X_1)$ is distributed identically to  $\Phi(A, B+\Gamma, C+1-\Gamma)$. In particular, for any $\delta\in[0,1]$ such that $p~\ge~ \frac{8\ln(2/\delta)}{n}$, $\shuffler(X_0)$ is $(\eps, q\delta)$-indistinguishable from $\shuffler(X_1)$, where
\begin{equation}\label{newproofeps}
\eps \leq \ln\left(1+q\left(\frac{4\sqrt{2\log(4/\delta)}}{\sqrt{pn}}+\frac{4}{pn}\right)\right).\end{equation}
\end{restatable}

The reduction to analysis of the divergence between a pair of multinomial distributions in the proof of Theorem~\ref{higheps0new} is achieved using the same argument as the one we used to prove Theorem~\ref{shuffletobinoms}. The analysis of the resulting divergence
relies on decomposing each of the resulting multinomial distributions into a mixture of two distributions. We then show that distributions we analyzed when proving Lemma~\ref{binomials} can be postprocessed into pairs of components of these mixtures. This allows us to show that all components of these mixtures are $(\eps',\delta')$-indistinguishable.  Joint convexity can then be used to obtain the bound. The details of the proof appear in Appendix~\ref{specificproof}.

The main strength of Theorem~\ref{higheps0new} is that it allows us to prove tighter bounds than Theorem~\ref{higheps0} for specific local randomizers. In the next section, we will show how Theorem~\ref{higheps0new} can be used to provide a tighter analysis of $k$ randomized response. However, we can also use it to rederive the bound from Theorem~\ref{higheps0}, although with a slightly worse constant factor. The decomposition in Lemma \ref{lem:kov} implies that for any $\eps_0$-DP local randomizer Equation~\eqref{TVeq} holds with $q=\frac{e^{\eps_0}-1}{e^{\eps_0}+1}$. That is, given a sequence of $\Aldp[i]$ of $\eps_0$-DP local randomizers, and neighbouring datasets $X_0$ and $X_1$ such that $x_0^0\neq x_1^1$, there exist distributions $\Q^{(i)}(z_{1:i-1},x_1^0)$ and $\Q^{(i)}(z_{1:i-1},x_1^1)$ that satisfy
equations \eqref{eq:decomp0} and \eqref{eq:decomp1}. Thus, by setting
\[\qzi = \frac{1}{2}\left(\Q^{(i)}(z_{1:i-1},x_1^0) + \Q^{(i)}(z_{1:i-1},x_1^1) \right) \mbox{ and } \qo{x_1^b} = \Q^{(i)}(z_{1:i-1},x_1^b) \]
we get that for $b\in\zo$ and $q=\frac{e^{\eps_0}-1}{e^{\eps_0}+1}$, \[\Aldp[i](z_{1:i-1}, x_1^b) = (1-q)\qzi+q \qo{x_1^b}.\]
Further, Equation~\eqref{mixture} implies that for any $x\in\mathcal{D}$ and $b\in \zo$,
\alequn{
\Aldp[i](z_{1:i-1},x) &= \frac{1}{2e^{\eps_0}}\Q^{(i)}(z_{1:i-1},x_1^0)+\frac{1}{2e^{\eps_0}}\Q^{(i)}(z_{1:i-1},x_1^1)+\left(1-\frac{1}{e^{\eps_0}}\right)\lo^{(i)}(z_{1:i-1},x)\\
&= \frac{1}{3e^{\eps_0}} \qzi + \frac{1}{3e^{\eps_0}}\qo{x_1^b} + \frac{1}{3e^{\eps_0}}\qo{x_1^{1-b}} + \left(1-\frac{1}{e^{\eps_0}}\right)\lo^{(i)}(z_{1:i-1},x).}
Letting $q=\frac{e^{\eps_0}-1}{e^{\eps_0}+1}$ and $p=e^{-\eps_0}/3$ and assuming $\eps_0\le\ln\left(\frac{n}{24\ln(2/\delta)}\right)$ (which implies $p\ge \frac{8\ln(2/\delta)}{n}$),  Equation~\eqref{newproofeps} gives the bound
\[\eps \leq \ln\left(1+\frac{e^{\eps_0}-1}{e^{\eps_0}+1}\left(\frac{4\sqrt{6e^{\eps_0}\log(4/\delta)}}{\sqrt{n}}+\frac{12e^{\eps_0}}{n}\right)\right),\]
which matches the bound stated in Equation~\eqref{epsbound} of Theorem~\ref{higheps0} up to a factor of $\sqrt{3/2}$.

\subsection{$k$-Randomized Response}

For any $k\in\mathbb{N}$ and $\eps_0 >0$, the $k$-randomized response $\kRR\colon [k]\to[k]$ is defined as \[\kRR(x) = \begin{cases} x & \text{with probability } \frac{e^{\eps_0}-1}{e^{\eps_0}+k-1}\\  y\sim \unif{[k]} & \text{with probability } \frac{k}{e^{\eps_0}+k-1}\end{cases},\] where $\unif{[k]}$ is the uniform distribution over $[k]$. That is, with probability $\frac{e^{\eps_0}-1}{e^{\eps_0}+k-1}$ the true data point is reported, and otherwise a random value is reported.

\begin{restatable}{corollary}{rkRRthm}\label{kRRthm} For a domain $\mathcal{D}$, let $\Aldp[i]:\out[1]\times\cdots\times\out[i-1]\times\mathcal{D}\to\out[i]$ for $i\in[n]$ (where $\out[i]$ is the range space of $\Aldp[i]$) be a sequence of algorithms such that $\Aldp[i](z_{1:i-1}, \cdot)$ is a $\eps_0$-DP local randomizer for all values of auxiliary inputs $z_{1,i-1}\in\out[1]\times\cdots\times\out[i-1]$. Let $\shuffler:\mathcal{D}^n\to\out[1]\times\cdots\times \out[n]$ be the algorithm that given a dataset $x_{1:n}\in\mathcal{D}^n$, samples a uniformly random permutation $\pi$, then sequentially computes $z_i=\Aldp[i](z_{1:i-1}, x_{\pi(i)})$ for $i\in[n]$ and outputs $z_{1:n}$. Assume that for some $k\in\mathbb{N}$ we have that for all $i\in[n]$, there exists a function $f^{(i)}:\out[1]\times\cdots\times\out[i-1]\times\mathcal{D}\to[k]$ such that $\Aldp[i](z_{1:i-1}, x)=\kRR(f^{(i)}(z_{1:i-1},x))$. Then for $\delta\in[0,1]$ such that such that $\eps_0\le\log(\frac{n}{16\log(2/\delta)})$,
$\shuffler$ is $(\eps,\delta)$-DP where
\begin{equation}\label{krrepsbound}
\eps\le \log\left(1+(e^{\eps_0}-1)\left(\frac{4\sqrt{2(k+1)\log(4/\delta)}}{\sqrt{(e^{\eps_0}+k-1)kn}}+\frac{4(k+1)}{kn}\right)\right)
\end{equation}
\end{restatable}

Notice that when $k$ is small, this bound matches that given in Theorem~\ref{higheps0}, but when $k$ and $\eps_0$ are large, $\eps$ scales like $\frac{e^{\eps_0}}{\sqrt{kn}}$. The proof of Corollary~\ref{kRRthm} can be found in Appendix~\ref{appendixkrr}. The key observation is that if we let $\nu$ be the uniform distribution on $[k]$, then for any $x\in[k]$, \[\kRR(x) = \frac{k}{e^{\eps_0}+k-1}\nu+\frac{e^{\eps_0}-1}{e^{\eps_0}+k-1}\indicator_x,\] where $\indicator_x$ is the distribution that always outputs $x$.

\section{Applications}
\subsection{Frequency and Distribution Estimation}\label{optimal}
In this section we show how our amplification result immediately implies an algorithm for frequency and distribution estimation in the shuffle model with nearly optimal utility/privacy trade-off in both the local and the central models and low communication. For an integer $k$, a discrete distribution over $[k]$ is described by a vector $p$ such that $\sum_{i \in [k]} p_i =1$ and $p_i \geq 0$ for all $i\in [k]$. We also write $x \sim p$ to say that $x$ is drawn from the distribution defined by $p$.
We say that an algorithm $\A$ achieves $\ell_2$ error of $\alpha$ for distribution estimation if for all distributions over $p$, the algorithm given $n$ i.i.d.~ samples from $p$ outputs  a vector $\hat p$ such that $$\E[\|p - \hat p\|_2^2] \leq \alpha^2 ,$$ where the expectation is over the randomness of samples and the algorithm.

The problem of distribution estimation is closely related to the problem of frequency (or histogram) estimation. In this problem, given a dataset $X\in [k]^n$ the goal is to estimate the vector of empirical frequencies $p(X)$ of elements held by the clients (namely, for all $i \in [k]$, $p(X)_i = \frac{1}{n} |\{ j \in [k] \ | \ x_j = i\}|$).  It is one of the most well-studied problems in private data analysis and the optimal utility/privacy trade-offs for this problem are well-understood \cite{Dwork:2006} (although the bounds are typically stated only for the $\ell_\infty$ error in the estimation).

For simplicity we will again focus on the $\ell_2$ error or $\sqrt{\E[\|p(X) - \hat p\|_2^2]}$ (with expectation being only with respect to the randomness of the algorithm that outputs the estimate $\hat p$). It is easy to see that
$$\E_{X\sim p^n}\left[\|p - p(X)\|_2^2\right] \leq \frac{1}{n} ,$$ and therefore an algorithm for frequency estimation implies an algorithm for distribution estimation with error larger by at most $1/\sqrt{n}$.

A number of algorithms for frequency and distribution estimation in the local model have been developed \citep{hsu2012distributed,erlingsson2014rappor,bassily2015local,kairouz2016discrete,wang2016mutual,WangBLJ:17,ye2018optimal,Acharya:2019,acharya2019communication,bun2019heavy,bassily2020practical,chen2020breaking,FeldmanTalwar21}. Recent work focuses on achieving (asymptotically) optimal accuracy in the $\eps_0 > 1$ regime and low communication. 
In particular, \citet{Acharya:2019} give an efficient, low communication ($\log k + 2$ bits per user) $\eps_0$-DP local algorithm that is asymptotically optimal for all $\eps_0$ (their result is stated for the distribution estimation problem but also applies to the frequency estimation problem). A related algorithm with similar theoretical guarantees but somewhat better empirical performance is given in \citep{chen2020breaking}.

\begin{theorem}[\citep{Acharya:2019,chen2020breaking}]\label{thm:hr}
For every positive integer $k$ and $\eps_0>0$, there exists an $\eps_0$-LDP protocol for frequency estimation that outputs a vector of frequencies $\hat p$ such that for every dataset $X \in [k]^n$,
 \[\E \left[\|p(X) - \hat p\|_2^2\right]= O \left( \frac{1}{n} + \frac{ k e^{\eps_0}}{n (e^{\eps_0} -1)^2} \right) .\]
The server-side decoding takes time $\tilde O(n+k)$ and each local randomizer sends a single message of length $\log k +2$ and takes time $O(\log k + \eps_0)$.
\end{theorem}
The bounds achieved by this algorithm are asymptotically optimal for LDP \citep{ye2018optimal} when $\eps_0 \leq \log k$. (A different setting of parameters can be used to achieve optimality in the regime when $\eps_0 > \log k$ but we omit this regime as it does not appear to be practically relevant.) The best accuracy for low-communication protocols, specifically
 \[\E \left[\|p(X) - \hat p\|_2^2\right]= \frac{1}{n} + \frac{4 k e^{\eps_0}}{n (e^{\eps_0} -1)^2}\]
 is achieved by an algorithm in a recent work of \citet{FeldmanTalwar21} albeit at the expense of slower server-side decoding time.

Using the fact that our amplification bound has optimal dependence on $\eps_0$, we are able to leverage any of the asymptotically optimal results for local DP to immediately obtain a low communication, single round algorithm in the shuffle model whose error matches the optimal error for distribution estimation in the central model up to a $O(\sqrt{\log(1/\delta)})$ factor. For frequency estimation the error in the central model is optimal up to the $1/\sqrt{n}$ additive term (that is comparable to the statistical error) which is known to be necessary for algorithms in the single message shuffle model~\citep{Ghazi:2019}.

\begin{theorem}\label{shufflehist}
	For every positive integer $k$ and $\eps,\delta  \in (0,1)$, there exists an $(\eps,\delta)$-DP protocol
 for frequency estimation that outputs a vector of frequencies $\hat p$ such that for every dataset $X \in [k]^n$,
 \[\E \left[\|p - \hat p\|_2^2\right]= O\left(\frac{k\log(1/\delta)}{(\eps n)^2} + \frac{1}{n}\right) .\]
	The server-side decoding takes time $\tilde O(n+k)$. Each local randomizer sends a single message of length $\log k +2$ and satisfies $\eps_0$-LDP for
	$$\eps_0= \begin{cases}  \frac{\eps \sqrt{n}}{16\sqrt{\log(1/\delta)}}  & \text{when }  \eps \leq \sqrt{\log(1/\delta)/n};\\  \log\left(\frac{\eps^2 n}{100\log(1/\delta)}\right) & \text{when } \eps \in (\sqrt{\log(1/\delta)/n},1).\end{cases}  $$
\end{theorem}
\begin{proof} For a local randomizer $\AldpNI$, let $\shuffler(\AldpNI)$ be the algorithm that given a data set $x_{1:n}\in\cD^n$, samples a uniform random permutation $\pi$ over $[n]$, then outputs $(\AldpNI(x_{\pi(1)}), \cdots, \AldpNI(x_{\pi(n)}))$.

By Theorem~\ref{higheps0}, setting
	$$\eps_0= \begin{cases}  \frac{\eps \sqrt{n}}{16 \sqrt{\log(1/\delta)}}  & \text{when }  \eps \leq \sqrt{\log(1/\delta)/n};\\  \log\left(\frac{\eps^2 n}{100 \log(1/\delta)}\right)  & \text{when } \eps \in (\sqrt{\log(1/\delta)/n},1).\end{cases}  $$
ensure that for a local randomizer $\AldpNI$ that is $\varepsilon_0$-DP, the shuffled output $\shuffler(\AldpNI)$ is $(\varepsilon, \delta)$-DP.

According to Theorem~\ref{thm:hr}, there exists an $\eps_0$-DP local randomizer $\AldpNI$ and algorithm $f$ such that $f\circ\AldpNI$ has error 
 \[\E \left[\|p(X) - \hat p\|_2^2\right]= O \left( \frac{1}{n} + \frac{ k e^{\eps_0}}{n (e^{\eps_0} -1)^2} \right) .\]
Substituting the value of $\eps_0$ gives the claimed result.
\end{proof}

We note that for the frequency estimation problem in the shuffle model the first algorithm that achieves communication cost that is $o(k)$ while being asymptotically nearly optimal for the central model was given in~\cite{Ghazi:2019}. Their relatively involved multi-message protocol has communication cost of  $O(\log k \cdot \log n \cdot \log(1/(\eps\delta))/\eps^2)$. It also does not have local DP guarantees and is relatively inefficient computationally.

\subsection{Privacy Analysis of Private Stochastic Gradient Descent}
\label{sec:sgd}
In this section we show how our amplification result can be used to analyse private stochastic gradient descent (SGD), an algorithm for empirical risk minimization (ERM).
Suppose there exists a loss function $\ell$ and each user has a data point $x_i$. The goal is to minimize \[\mathcal{L}(\theta) = \sum_{i=1}^n \ell(\theta, x_i),\]
where $\theta \in \mathbb{R}^d$.
Gradient descent is a popular algorithm for solving ERM by taking iterative steps in the negative direction of the gradient. In \emph{stochastic} gradient descent, rather than compute each gradient of $\mathcal{L}$ at each step, a single data point is chosen, and the gradient $\nabla \ell(\cdot, x_i)$ is used to update the estimate. A differentially private version of stochastic gradient descent has been studied in \cite{song2013stochastic,Bassily:2014}, where the gradients are projected onto a bounded $\ell_2$ ball and zero-mean isotropic Gaussian noise is added at each iteration to preserve privacy.
 \citetall{Bassily:2014} analysed a version of private SGD where the data point at each round was chosen using sampling \emph{with} replacement. Our amplification result allows us to achieve almost the same utility using sampling \emph{without} replacement.
 Without replacement sampling is much more common and typically more efficient in practice. It also admits a simpler analysis of the generalization error (for example via an online-to-batch conversion~\cite{CesabianchiCG02}).
\begin{algorithm}\DontPrintSemicolon
  \textbf{Input:} $X$, privacy parameters $\eps_0, \delta_0$ and learning rate $\eta_{1:n}$\\
  Choose a random permutation $\pi$ of $[n]$\\
  Let $\sigma = \frac{1+\sqrt{2\log(1/\delta_0)}}{\eps_0}$\\
  Choose arbitrary initial point $\widetilde{\theta_0}$\\
  \For{$i=1:n$}{
    $b_i \sim \mathcal{N}({\bf 0}, \sigma^2 {\mathbb I}_d ) $\\
    $g_i = \nabla \ell(\widetilde{\theta}_{i-1}, x_{\pi(i)})$\\
    $\tilde{g}_i = g_i / \max(1, \|g_i\|_2)$ \tcp*{Clip to norm $1$}
    $\widetilde{\theta_i} = \widetilde{\theta}_{i-1}-\eta_i\cdot (g_i +b_i)$
  }
  {\bf return} $\widetilde{\theta_n}$
  \caption{Differentially Private Stochastic Gradient Descent} \label{PSGD}
\end{algorithm}
\begin{prop}\label{PSGDperm}
For any $\delta\in[0,1]$ such that $\eps_0\le\log(\frac{n}{16\log(2/\delta)})$, Algorithm~\ref{PSGD} is $(\eps, \delta+O(e^{\eps}\delta_0n))$-DP where
\begin{equation}\label{SGDbound}
    \eps=O\left((1-e^{-\eps_0})\left(\frac{\sqrt{e^{\eps_0}\log(1/\delta)}}{\sqrt{n}}+\frac{e^{\eps_0}}{n}\right)\right).\end{equation}
\end{prop}

\begin{proof} Note that if we let \[\Aldp[i](z_{1:i-1}, x) = \widetilde{\theta_i} = \widetilde{\theta_{i-1}}(z_{1:i-1})-\eta_t\cdot (\nabla \ell(\widetilde{\theta_{i-1}}(z_{1:i-1}), x_{\pi(i)})+b_i),\] then the output of Algorithm~\ref{PSGD} can be obtained by post-processing of the shuffled output $\shuffler(X)$. Since each $\Aldp[i](z_{1:i-1}, \cdot)$ is a $(\eps_0,\delta_0)$-DP local randomizer and $\eps_0\le\log(\frac{n}{16\log(2/\delta)})$, Theorem~\ref{approxDP} implies that Algorithm~\ref{PSGD} is $(\eps, \delta+(e^{\eps}+1)(1+e^{-\eps_0}/2)n\delta_0)$-DP, where $\eps$ is as given in Equation~\eqref{SGDbound}.
\end{proof}

For a single pass over the data and $\eps_0 > 1$, the privacy analysis in Proposition~\ref{PSGDperm} is tighter than that given by \cite{Bassily:2014}, who analyse sampling with replacement.
For any $\delta>0$, their analysis, which relies on privacy amplification by subsampling and advanced composition of differentially private algorithms, gives a privacy guarantee of $(\eps, n\delta+\frac{\delta_0}{n})$ where \begin{equation}\label{bassilybound}\eps~=~O\left(\sqrt{\log(1/\delta)n}\log(1+\frac{e^{\eps_0}-1}{n})+(e^{\eps_0}-1)\log(1+\frac{e^{\eps_0}-1}{n})\right)\approx \frac{\sqrt{\log(1/\delta)}(e^{\eps_0}-1)}{\sqrt{n}}.\end{equation}
For large $\eps_0$, our result above (Proposition~\ref{PSGDperm}) improves on this by a factor of $\Theta(\sqrt{e^{\eps_0}})$. We note that for multiple passes over the data, our analysis suffers an extra $\sqrt{\log(1/\delta)}$ factor in the privacy loss compared to Equation~\eqref{bassilybound}. \new{Work subsequent to \cite{Bassily:2014} has improved the privacy analysis for multiple passes using concentrated differential privacy to shave an additional $\sqrt{\log(1/\delta)}$  \citep{DLDP, Wang:2021}.} Our analysis technique for Algorithm~\ref{PSGD} is also the basis of the analysis of the optimization algorithms in \cite{Balle:2020}.
We also note that Proposition \ref{PSGDperm} can be easily extended to the batched version of SGD by viewing each batch of size $b$ as a single data element in $\mathcal D^b$.

\section{Numerical Results}\label{experiment}

Theorem~\ref{shuffletobinoms} provides an efficient method for numerically computing an amplification bound that is tighter than our closed-form bound. Our implementation is outlined in Appendix~\ref{pseudocode}, but the main component is numerically computing the indistinguishability bound for multinomial random variables from Lemma~\ref{binomials}. In this section, we provide numerical evaluations of the privacy amplification bound in a variety of parameter regimes. In Figure~\ref{graphs}, \texttt{Clones, Theoretical} is the bound presented in Theorem~\ref{higheps0} and \texttt{Clones} is the numerical version derived from  Theorem~\ref{shuffletobinoms}. Also, \texttt{BBGN'19} is the privacy amplification bound for general algorithms given in \cite{Balle:2019}\footnote{\texttt{BBGN'19} curves were produced using open source code released by \citetall{Balle:2019}, they correspond to the curve \texttt{(Bennett, Generic)} in \cite{Balle:2019}.} and \texttt{BKMTT'20} is the bound proven in~\cite{Balle:2020}. Finally, \texttt{CSUZZ'19, 2RR} is the closed form amplification bound for binary randomized response proved in \cite{Cheu:2019} and \texttt{2RR, lower bound} is a lower bound on the shuffling privacy guarantee of binary randomized response that we compute directly. A description of our implementation of \texttt{2RR, lower bound} is found in Appendix~\ref{2RRimplementation}. We do not include the bounds from \cite{ErlingssonFMRTT19} in the comparison since \cite{Balle:2020} gives a tighter bound for the same analysis.

In all parameter regimes tested, \texttt{Clones} gives a tighter bound compared to \texttt{BBGN'19} as well as \texttt{BKMTT'20}. We can see in Figure~\ref{varyeps} that this effect is particularly pronounced when $\eps_0$ is large, which is the parameter regime where our closed-form bounds asymptotically improve over prior work. Being a lower bound for a particular algorithm, \texttt{2RR, lower bound} provides a lower bound on any general privacy amplification result. We can see in Figure~\ref{varyeps} that for large $n$, \texttt{Clones} closely tracks \texttt{2RR, lower bound}, particularly for small $\eps_0$. We note in all three graphs in Figure~\ref{graphs}, \texttt{Clones} should be monotone. The slight non-monotonicity results from some optimizations that speed up the computation at the cost of slightly looser upper bound. We defer details to Appendix~\ref{pseudocode}.

\begin{figure}
  \centering
  \begin{subfigure}{0.47\textwidth}
  \includegraphics[width=1\textwidth]{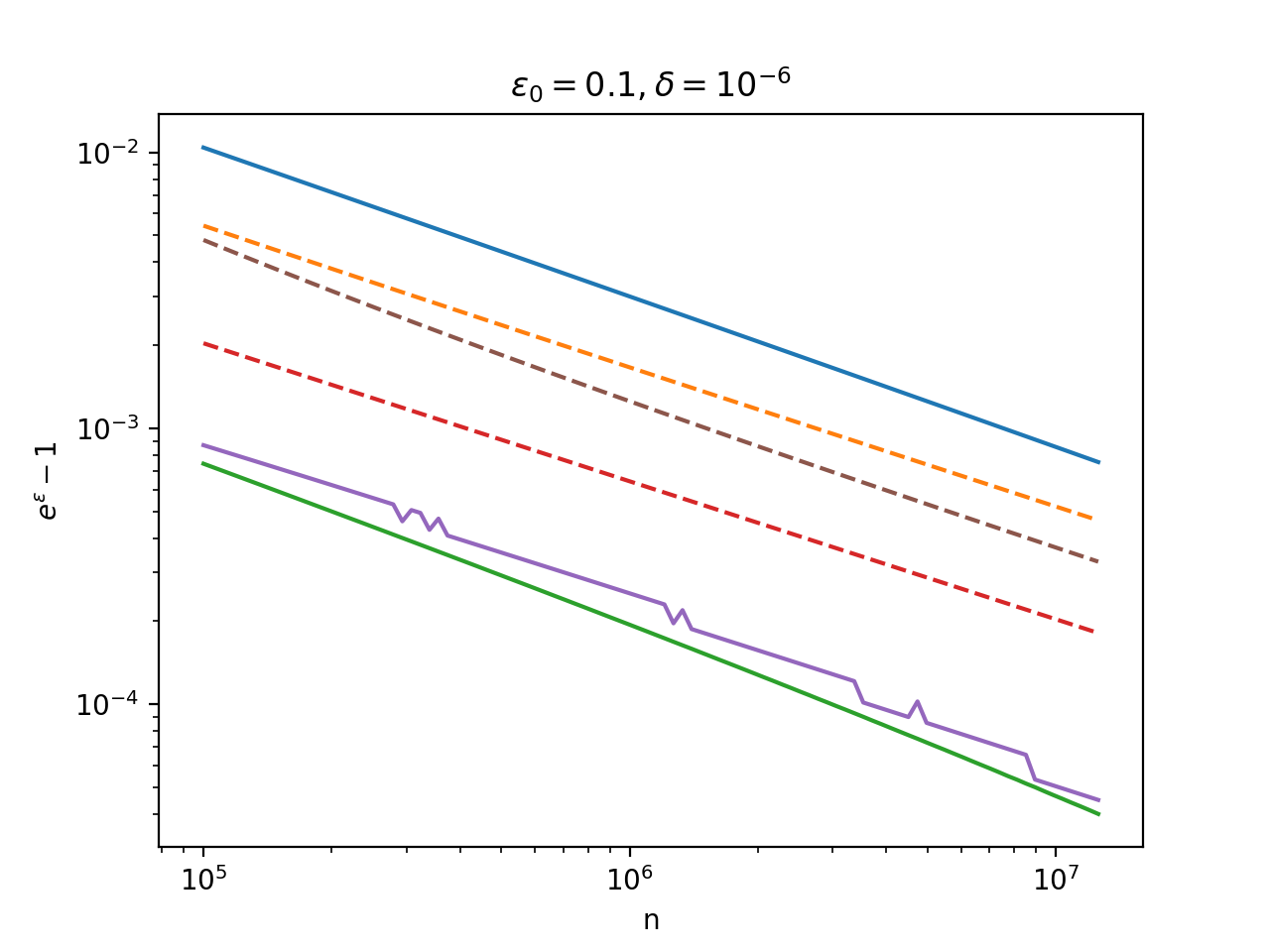}
  \caption{Comparison of central $\eps$ for fixed $\eps_0=0.1$ and $\delta=10^{-6}$ and $n$ ranging between $10^5$ and $10^7$.}\label{varynsmalleps}
  \end{subfigure}
  \hspace{0.1in}
  \begin{subfigure}{0.47\textwidth}
  \includegraphics[width=1\textwidth]{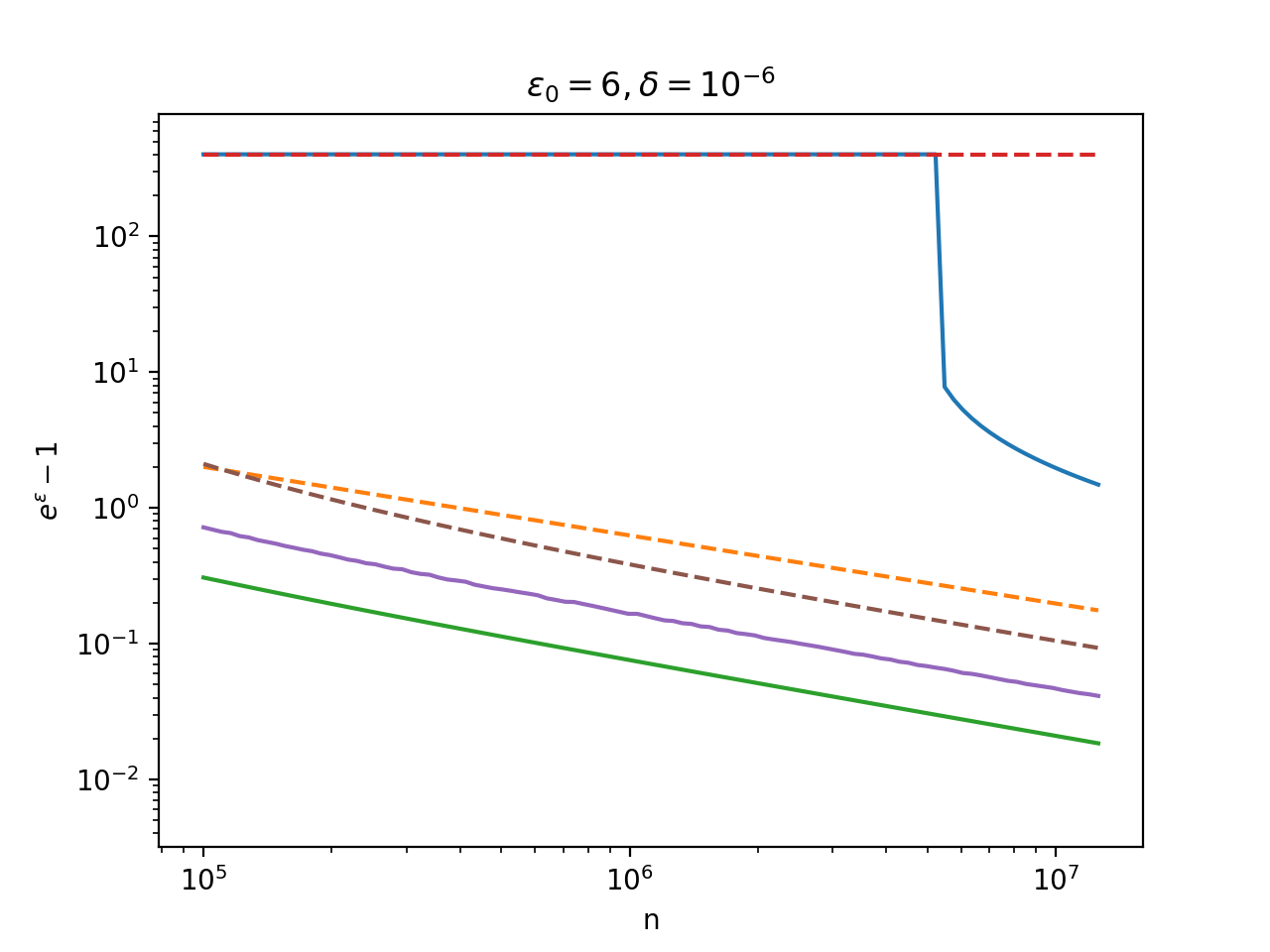}
  \caption{Comparison of central $\eps$ for fixed $\eps_0=6$ and $\delta=10^{-6}$ and $n$ ranging between $10^5$ and $10^7$.}\label{varynlargeeps}
  \end{subfigure}
  \begin{subfigure}{0.45\textwidth}
  \includegraphics[width=1\textwidth]{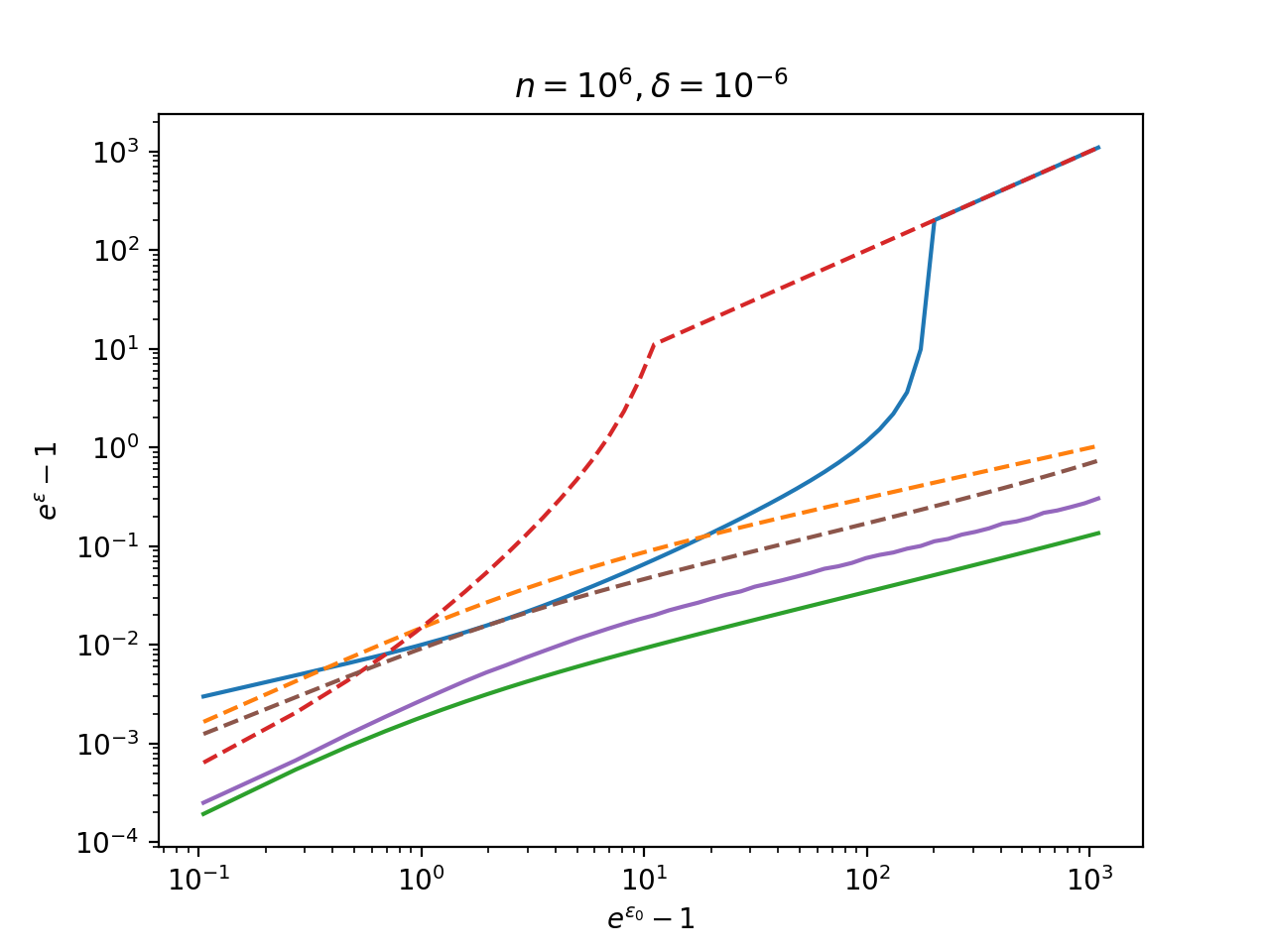}
  \caption{Comparison of central $\eps$ for fixed $n=10^6$ and $\delta=10^{-6}$ and $\eps_0$ ranging between $0.01$ and $8$.}\label{varyeps}
  \end{subfigure}
  \hspace{0.1in}
  \begin{subfigure}{0.45\textwidth}
  \begin{center}
  \hspace{0.3in}\includegraphics[width=0.7\textwidth]{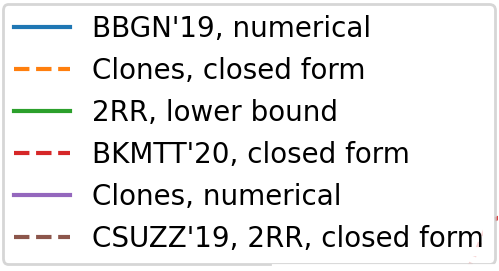}
  \vspace{0.4in}
  \end{center}
  \end{subfigure}
  \caption{Comparison of new privacy amplification by shuffling bounds given in this work to bounds given in \cite{Balle:2019} and \cite{Balle:2020} and to bounds specific to $\2RR$ given in \cite{Cheu:2019}. Dotted lines represent closed-form bounds, while solid lines represent bounds that were computed numerically.}\label{graphs}
\end{figure}

\subsection{Privacy Amplification in terms of R\'enyi DP}\label{renyiDPexperiments}

Indistinguishability of random variables is only one way of several possible ways to measure to similarity in the output distributions of neighboring datasets. The R\'enyi divergence, that captures the moment generating function of the privacy loss random variable, is an alternate measure resulting in R\'enyi differential privacy \cite{DworkR16,DLDP,mironov2017renyi,Bun:2016}.

\begin{definition}[R\'enyi divergence]
For two random variables $P$ and $Q$, the R\'enyi divergence of $P$ and $Q$ of order $\alpha>1$ is \[D^{\alpha}(P\|Q) = \frac{1}{\alpha-1}\log\mathbb{E}_{x\sim Q}\left(\frac{P(x)}{Q(x)}\right)^{\alpha}\]
\end{definition}

\begin{definition} For $\alpha>1$,
an algorithm $\mathcal{A}:\mathcal{D}^n\to\output$ is $(\eps(\alpha), \alpha)$-\emph{R\'enyi differentially private} (RDP) if for all neighboring databases $X$ and $X'$, $D^{\alpha}(\mathcal{A}(X)\|\mathcal{A}(X'))\le\eps(\alpha)$.
\end{definition}

Theorem~\ref{shuffletobinoms} implies that to bound the RDP guarantees of $\shuffler$ it suffices to compute the R\'enyi divergence between the two multinomial distributions $(A+\Delta,C-A+1-\Delta)$ and  $(A+1-\Delta,C-A+\Delta)$. Below we evaluate this approach numerically in Section~\ref{renyiDPexperiments}.

In Figure~\ref{graphsrenyi} we plot the R\'enyi privacy parameter for shuffling $n$ $\eps_0$-DP local randomizers as a function of the RDP order $\alpha$.
 \texttt{Clones} is computed by directly computing the R\'enyi divergence between the two multinomial distributions $(A+1-\Delta,C-A+\Delta)$ and $(A+\Delta,C-A+1-\Delta)$, providing an upper bound on the privacy amplification by shuffling for RDP. For comparison, we include a bound computed using Theorem~1 from \citep{Girgis:2021} which we label as \texttt{GDDSK'21}. (We remark that the results in this subsection were not included in the earlier version of our work.)  \texttt{2RR, lower bound} provides a lower bound on the shuffling privacy guarantee of binary randomized response, and hence a lower bound on any general amplification result for RDP. Figure~\ref{renyiDP} shows that privacy amplification is achieved for small/moderate moments, but the R\'enyi DP parameter $\eps(\alpha)$ approaches $\eps_0$ as $\alpha$ increases.

A key property of differentially private algorithms is that the composition of differentially private algorithms is differentially private. The advanced composition theorem quantifies the privacy guarantee after composing $T$ $(\eps,\delta)$-DP algorithms. When composing a large number of algorithms, \citep{DLDP} showed that tighter privacy guarantees can be obtained by computing the composition guarantees in terms of RDP and then converting back to approximate differential privacy. In Figure~\ref{graphsrenyi}, we plot the privacy guarantee for $T$ adaptively composed outputs of a shuffler with each shuffler operating on $n$, $\eps_0$-DP local randomizers as defined as in Theorem~\ref{higheps0}. We compare two methods for computing the resulting privacy guarantee. \texttt{Clones, via Approximate DP} computes the privacy guarantee of $\shuffler$ numerically in terms of approximate differential privacy, then used advanced composition \cite[Theorem 4.3]{Kairouz:2015} to compute the privacy guarantee of the composition. \texttt{Clones, via RDP} computes the privacy guarantee of $\shuffler$ numerically in terms of R\'enyi differential privacy, computes the privacy guarantee of the composition using the composition theorem for R\'enyi DP \cite{mironov2017renyi}, then converts to an approximate DP guarantee \cite[Proposition 12]{canonne2020discrete}. We also compare to the latter approach performed using the RDP privacy amplification result presented in \cite{Girgis:2021} (with parameters chosen as in \cite{Girgis:2021}). As can be seen from the results, our technique provides a significantly tighter upper bound.
\color{black}

\begin{figure}
  \centering
  \begin{subfigure}{0.47\textwidth}
  \includegraphics[width=1\textwidth]{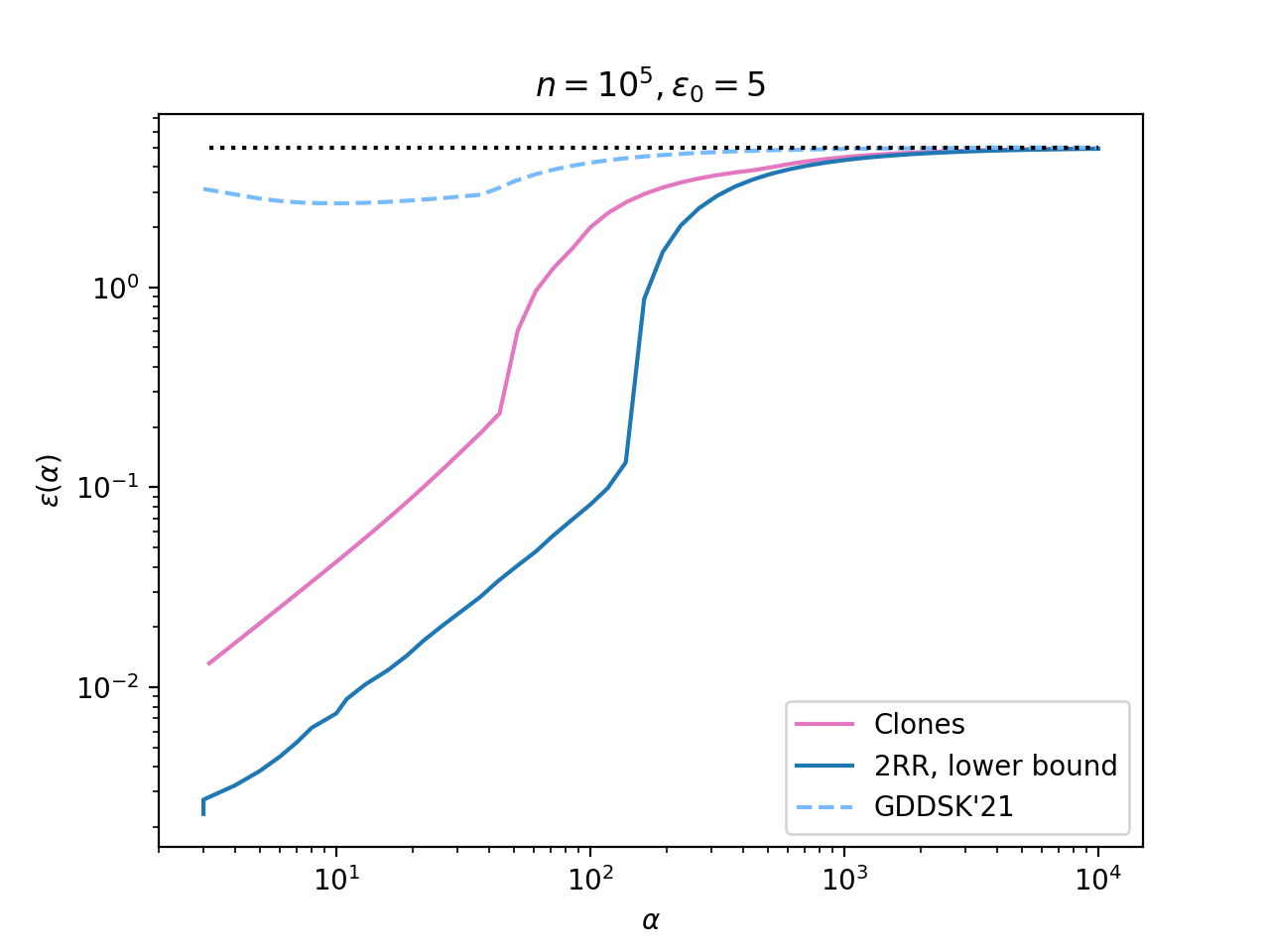}
  \caption{Comparison of the bounds on $\eps(\alpha)$ for RDP for fixed $\eps_0=5$ and $\alpha$ ranging between $3$ and $10^4$. The horizontal black, dotted line is at $\eps_0$.}\label{renyiDP}
  \end{subfigure}
  \hspace{0.1in}
  \begin{subfigure}{0.47\textwidth}
  \includegraphics[width=1\textwidth]{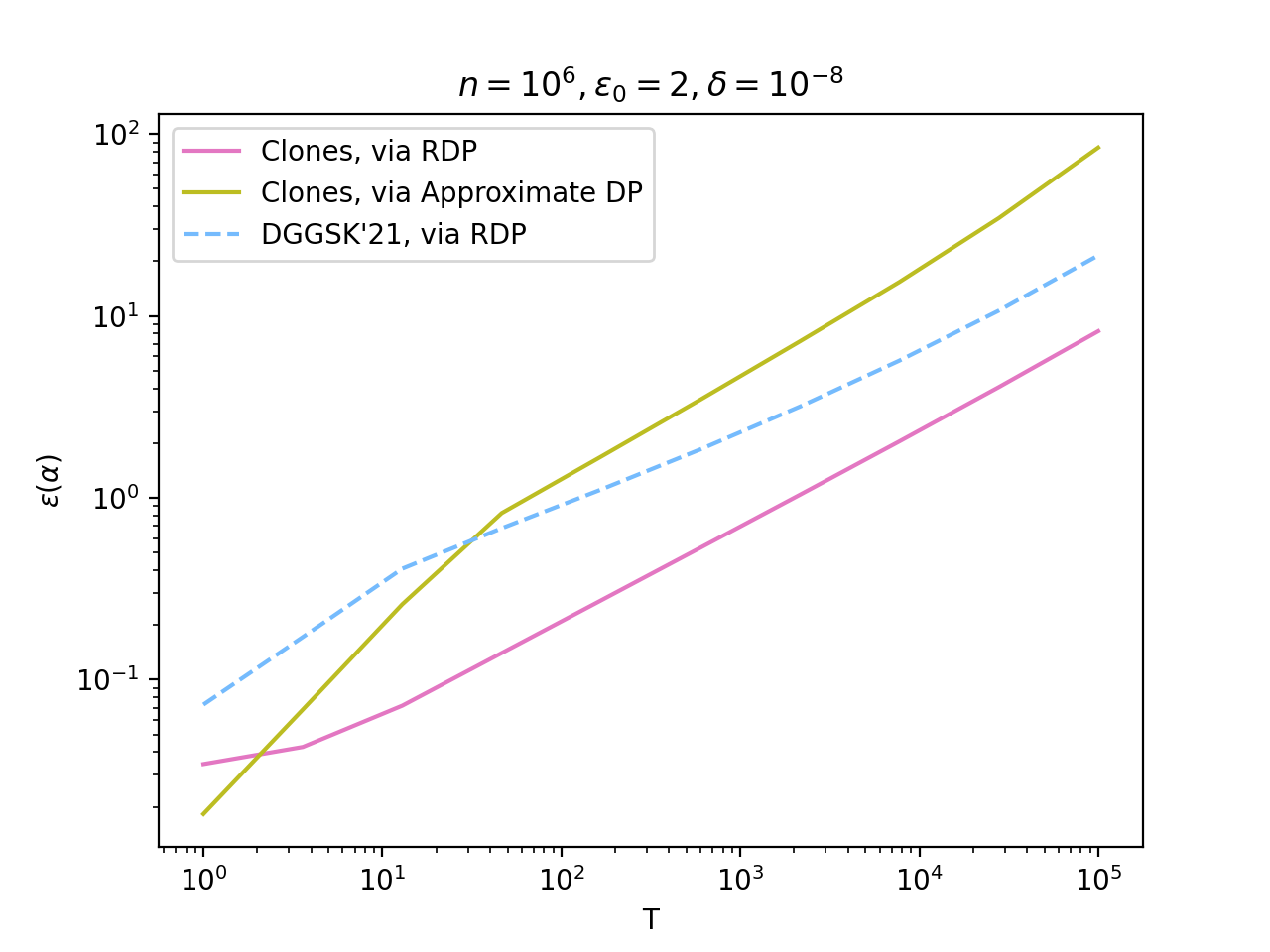}
  \caption{Approximate DP guarantees for composition of a sequence of $T$ results of amplification by shuffling.}\label{compositionrenyi-1}
  \end{subfigure}
    \begin{subfigure}{0.47\textwidth}
  \includegraphics[width=1\textwidth]{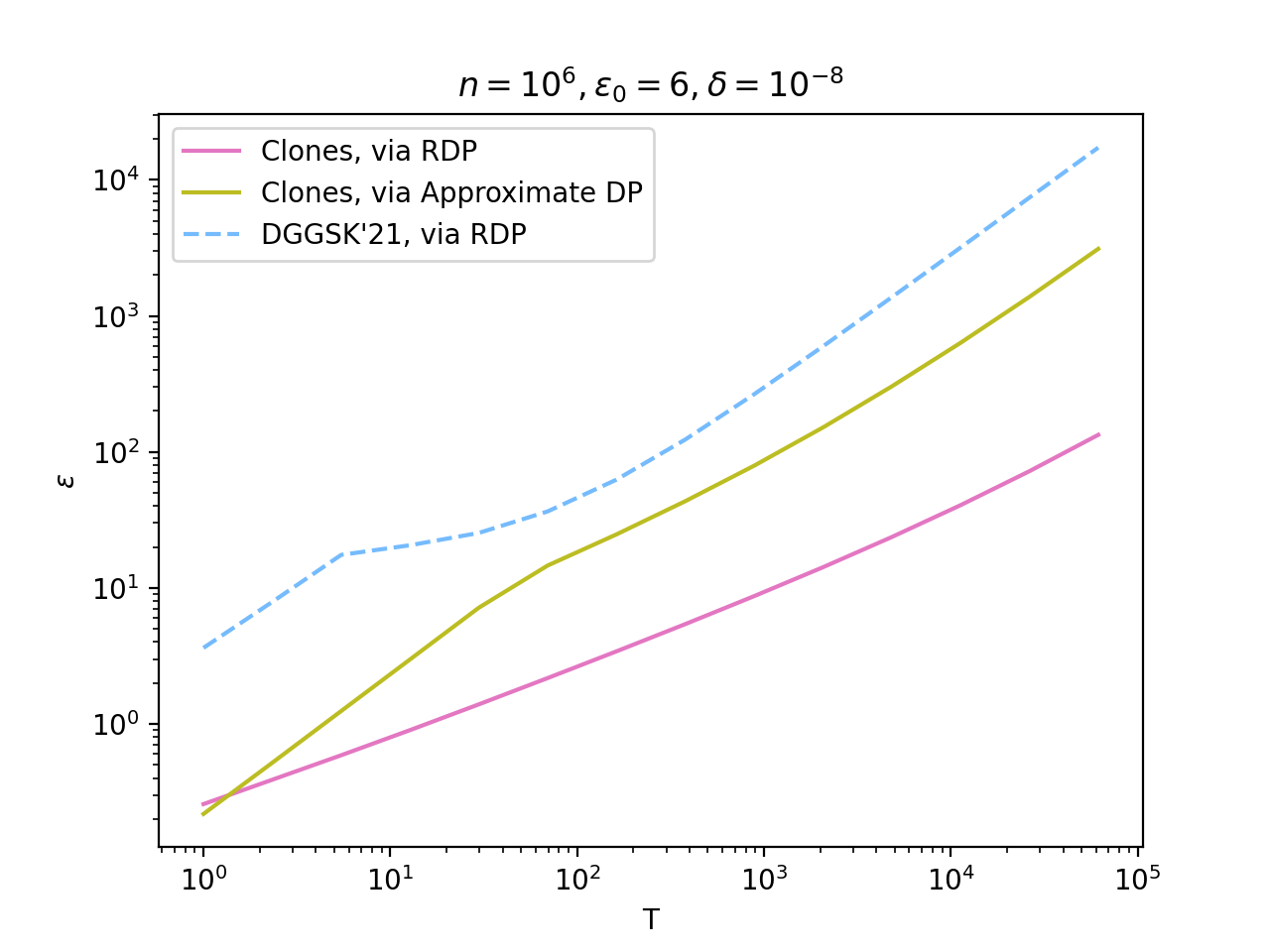}
  \caption{Approximate DP guarantees for composition of a sequence of $T$ results of amplification by shuffling.}\label{compositionrenyi-2}
  \end{subfigure}
      \begin{subfigure}{0.47\textwidth}
  \includegraphics[width=1\textwidth]{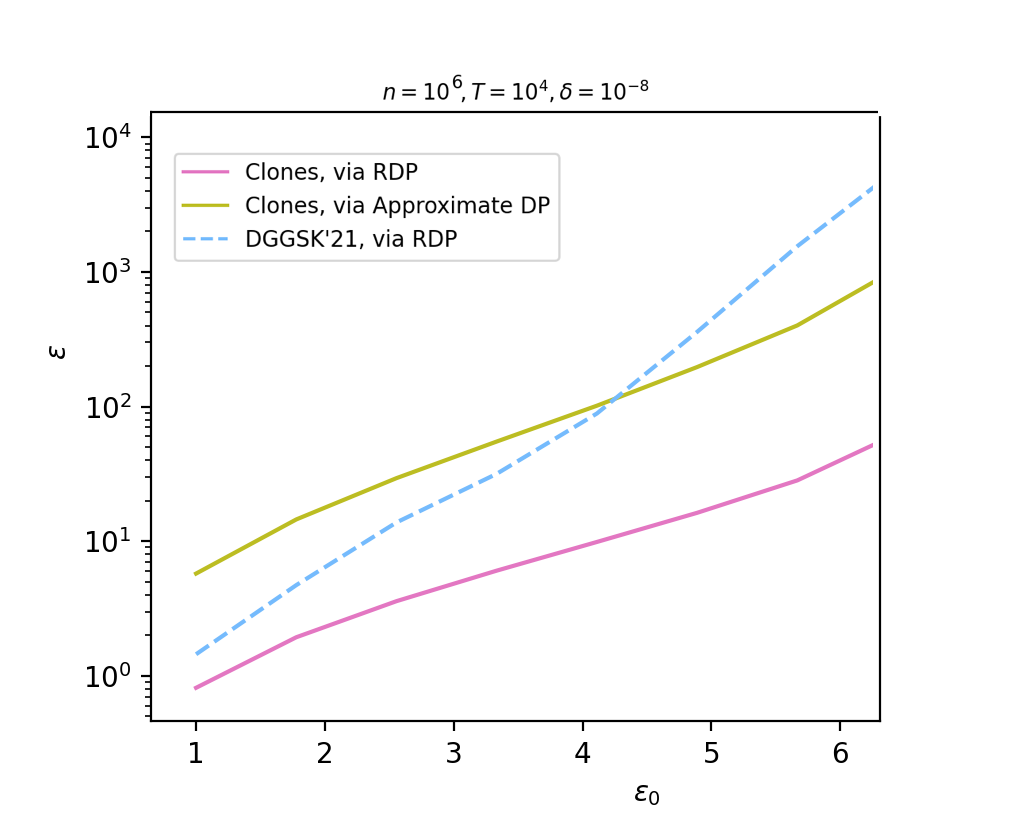}
  \caption{Approximate DP guarantees for composition of a sequence of $10,000$ results of amplification by shuffling as a function of $\eps_0$.}\label{compositionrenyi-3}
  \end{subfigure}
  \caption{RDP guarantees for privacy amplification by shuffling and implications for approximate DP bounds for composition.   }
  \label{graphsrenyi}
\end{figure} 

\printbibliography

\appendix

\newpage

\section{Proof of Lemma~\ref{binomials}}\label{binomialsproof}

\rbinomials*
To prove Lemma~\ref{binomials} we first upper bound the divergence between a pair of slightly simpler distributions.
\begin{lemma}\label{binomials-aux} Let $p\in[0,1]$ and $n\in\mathbb{N}$ be such that $p~\ge~ \frac{16\ln(2/\delta)}{n}$, and let \[\eps=\log\left(1+\frac{\sqrt{64\log(4/\delta)}}{\sqrt{pn}}+\frac{8}{pn}\right).\]
Consider the process where we sample $C\sim \bin(n-1, p)$ and $A\sim\bin(C, 1/2)$. Let $P=(A+1,C-A)$ and $Q=(A,C-A+1)$, then
\[\Pr_{(a,c)\sim P}\left(-\eps\le\ln\frac{\Pr(P=(a,c))}{\Pr(Q=(a,c))}\le \eps\right)\ge 1-\delta \text{.   and    } \Pr_{(a,c)\sim Q}\left(-\eps\le\ln\frac{\Pr(P=(a,c))}{\Pr(Q=(a,c))}\le \eps\right)\ge 1-\delta.\]
In particular, $P$ and $Q$ are $(\eps, \delta)$- indistinguishable.
\end{lemma}
\begin{proof}
Using a Chernoff bound and Hoeffding's inequality, with probability $\delta$ since $p\ge\frac{3\log(4/\delta)}{n}$ both
\begin{equation}\label{chernoff}
|C-p n|\le \sqrt{3p n\log(4/\delta)}\;\; \text{ and }\;\; |A-C/2|\le\sqrt{\frac{C}{2}\log(4/\delta)}.
\end{equation}
Now, note that $p\ge \frac{16\ln(4/\delta)}{n}$ implies that $\sqrt{3p n\log(4/\delta)}\le p n/2$ and, $\sqrt{\frac{C}{2}\log(4/\delta)}\le C/4$.
Given a specific output $(a,b)$,
\begin{align*}
\Pr(P=(a,b)) &= \Pr(C=a+b-1)\cdot\Pr(A=a-1|C=a+b-1)\\
&= \Pr(C=a+b-1)\cdot\frac{a}{b}\cdot\Pr(A=a|C=a+b-1)\\
&= \frac{a}{b}\cdot\Pr(Q=(a,b))
\end{align*}
Now, if $A$ and $C$ satisfy equation (\ref{chernoff}) then
\begin{align*}
\frac{a}{b}&=\frac{A+1}{C-A}\\
&\le \frac{C/2+\sqrt{C/2\log(4/\delta)}+1}{C/2-\sqrt{C/2\log(4/\delta)}}\\
&\le 1+\frac{\sqrt{32\log(4/\delta)}}{\sqrt{C}}+\frac{4}{C} \\
&\le \textcolor{black}{1+\frac{\sqrt{32\log(4/\delta)}}{\sqrt{p n-\sqrt{3p n\log(4/\delta)}}}}+\frac{8}{p n}\\
&\le 1+\frac{\sqrt{64\log(4/\delta)}}{\sqrt{p n}}+\frac{8}{pn}\\
\end{align*}
Further,
\begin{align*}
\frac{b}{a}&= \frac{C-A}{A+1}\\
&\le \frac{C/2+\sqrt{C/2\log(4/\delta)}}{C/2-\sqrt{C/2\log(4/\delta)}+1}\\
&\le \frac{C/2+\sqrt{C/2\log(4/\delta)}+1}{C/2-\sqrt{C/2\log(4/\delta)}}\\
&\le 1+\frac{\sqrt{64\log(4/\delta)}}{\sqrt{p n}}+\frac{8}{pn}\\
\end{align*}

Therefore, setting $\eps=\log(1+\frac{\sqrt{64\log(4/\delta)}}{\sqrt{pn}}+\frac{8}{pn})$ gives that $P$ and $Q$ are $(\eps, \delta)$ indistinguishable.

Further, \begin{equation}\Pr_{(a,c)\sim P}\left(-\eps\le\ln\frac{\Pr(P=(a,c))}{\Pr(Q=(a,c))}\le \eps\right)\ge 1-\delta\end{equation} and \begin{equation}\Pr_{(a,c)\sim Q}\left(-\eps\le\ln\frac{\Pr(P=(a,c))}{\Pr(Q=(a,c))}\le \eps\right)\ge 1-\delta.\end{equation}

\end{proof}

We can now complete the proof of Lemma~\ref{binomials} by using the advanced joint convexity of the hockey-stick divergence.
\begin{proof}[Proof of Lemma~\ref{binomials}]
Let $P_0=(A+1,C-A)$ and $Q_0=(A,C-A+1)$ and let $\rho_0$ and $\rho_1$ denote their respective probability distributions. Then, by Lemma~\ref{binomials-aux} for $p= e^{-\eps_0}$,
\begin{equation}\label{indistinguishableDP}
\Pr_{z\sim \rho_0}\left(-\eps\le\ln\frac{\Pr(P_0=z)}{\Pr(Q_0=z)}\le\eps\right)\ge1-\delta \text{   and    } \Pr_{z\sim \rho_1}\left(-\eps\le\ln\frac{\Pr(P_0=z)}{\Pr(Q_0=z)}\le\eps\right)\ge1-\delta\end{equation}
for
\[
\eps\le \log\left(1+\frac{8\sqrt{\log(4/\delta)}}{\sqrt{pn}}+\frac{8}{pn}\right)
\]

Now, note that if we let $T=\frac{1}{2}(P_0+Q_0)$ then we obtain that

\[
P = \frac{2}{e^{\eps_0}+1} T + \frac{e^{\eps_0}-1}{e^{\eps_0}+1}P_0 \;\;\text{and}\;\;Q = \frac{2}{e^{\eps_0}+1} T+\frac{e^{\eps_0}-1}{e^{\eps_0}+1} Q_0.
\]

Further, for any $\eps>0$, by the convexity of the hockey-stick divergence, \[\max\{\dalpha{e^{\eps}}(T\|P_0)), \dalpha{e^{\eps}}(Q_0\|T))\}\le \dalpha{e^{\eps}}(P_0\|Q_0).\]

Thus, by Lemma~\ref{advancedjointconvexity}, $P$ and $Q$ are $(\eps,p\delta)$-indistinguishable where \[\eps\le\log\left(1+\frac{e^{\eps_0}-1}{e^{\eps_0}+1}\left(\frac{8\sqrt{\log(4/\delta)}}{\sqrt{pn}}+\frac{8}{pn}\right)\right) = \log\left(1+\frac{e^{\eps_0}-1}{e^{\eps_0}+1}\left(\frac{8\sqrt{e^{\eps_0}\log(4/\delta)}}{\sqrt{n}}+\frac{8e^{\eps_0}}{n}\right)\right).
\]
\end{proof}

\section{Tails of the Privacy Loss}
\label{app:verify-tails}
\begin{figure}
    \centering
    \includegraphics{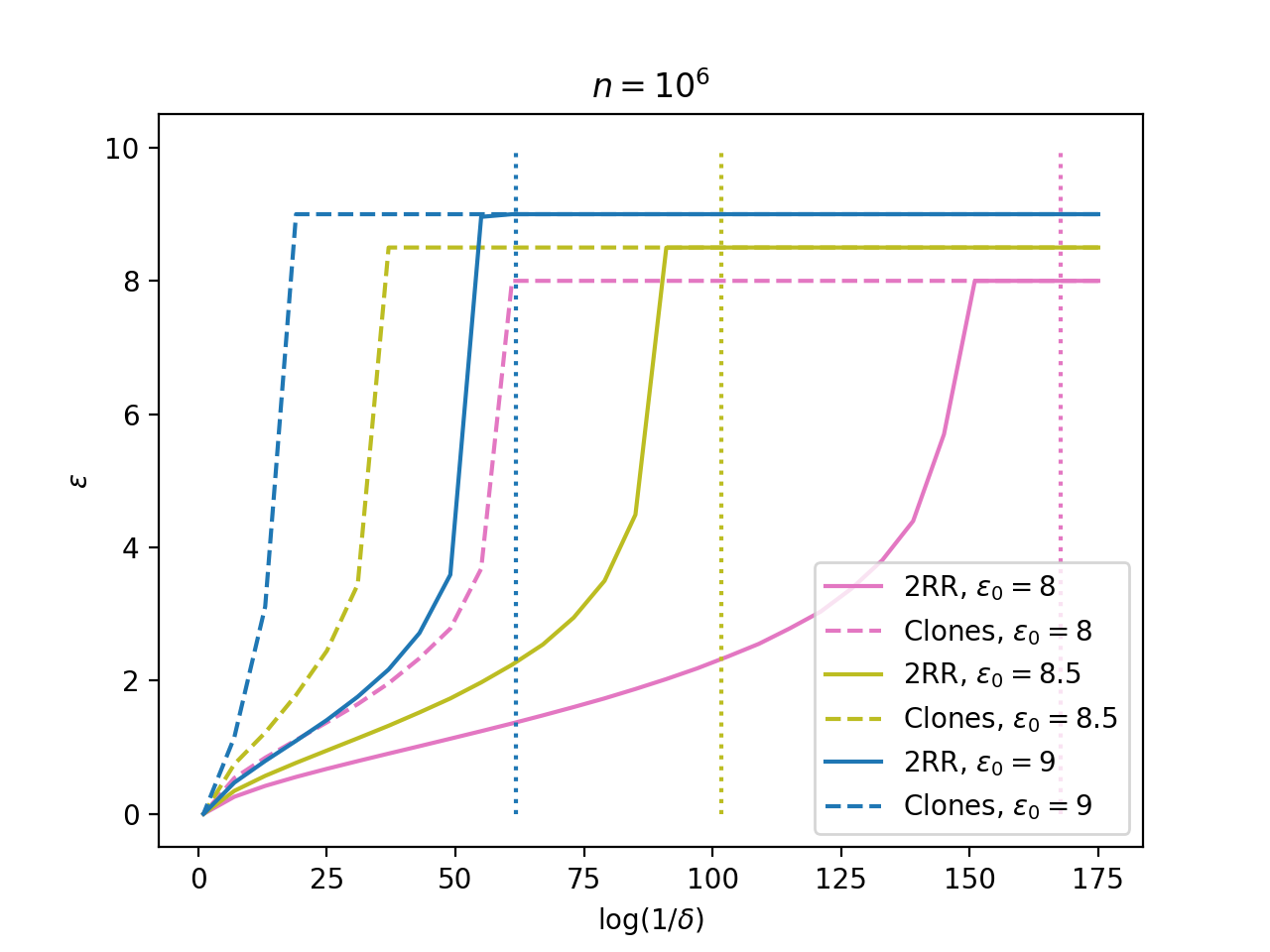}
    \caption{Bounds on the privacy amplification parameter as a function of $\delta$.}
    \label{varydelta}
\end{figure}

Recall from Theorem~\ref{higheps0} that the closed form bound for privacy amplification by shuffling was only valid when $\log(2/\delta)\le \frac{ne^{\eps_0}}{16}$. This condition arises in the proof when obtaining concentration bounds for the binomial, so a natural question is whether it is an artifact of the proof, or inherent. That is, do we get amplification when $\log(2/\delta)\ge \frac{ne^{\eps_0}}{16}$? In Figure~\ref{varydelta}, we show the tail of privacy loss as $\delta$ varies. The solid lines show the exact computation of the $\eps$ for shuffled binary randomised response. The dashed lines show the general upper bound obtained through the numerical computation derived from  Theorem~\ref{shuffletobinoms}, as in Section~\ref{experiment}. The horizontal, dotted lines mark $\log(1/\delta)=\frac{ne^{-\eps_0}}{2}$. Notice that in all three settings of $\eps_0$, the shuffled privacy loss undergoes a sharp transition from $\eps<\eps_0$ to $\eps=\eps_0$. Further, this transition point closely aligns with $\log(1/\delta)=\frac{ne^{-\eps_0}}{2}$, indicating that amplification is not achieved when $\log(1/\delta)=\Omega(ne^{-\eps_0})$.

\section{Proof of Proposition~\ref{approxDP}}
\label{app:approxproofs}

\rapproxtopure*

\begin{figure}
\centering
\includegraphics[scale=0.5]{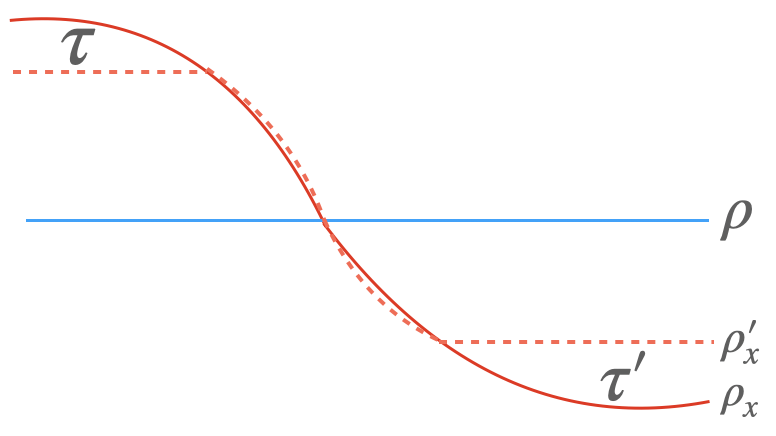}
\caption{Graphical representation of $\reference$, $\reference_x$ and $\reference_x'$ from the proof of Lemma~\ref{approxtopure}.}\label{fig:PQ}
\end{figure}
\begin{proof}[Proof of Lemma~\ref{approxtopure}]
Let $x\in\mathcal{D}$, $\reference$ be the reference distribution and $\reference_x$ be the probability density function of $\lr(x)$. Thus $\reference_x$ is $(\eps, \delta)$-indistinguishable from $\reference$. For $a,b,c\in \mathbb{R}$ such that $b<c$, let $[a]_b^c=\max\{\min\{a,c\},b\}$ be $a$ projected onto the interval $[b,c]$. Define a function $\reference_x'$ by \[\reference_x'(y) = [\reference_x(y)]_{e^{-\eps}\reference(y)}^{e^{\eps}\reference(y)}\] so that for all $y$,
\begin{equation}\label{ratio}
\ln\frac{\reference(y)}{\reference_x'(y)}\in[-\eps, \eps].
\end{equation}
Further, by the definition of $(\eps,\delta)$-indistinguishability,
\begin{equation}\label{TV}
\max\left\{\int_y \max\{\reference_x(y)-\reference_x'(y), 0\} dy, \int_y \max\{\reference_x'(y)-\reference_x(y), 0\} dy\right\}\le\delta.
\end{equation}

Now, $\reference_x'$ is not necessarily a distribution since $\int_y \reference_x'(y) dy$ is not necessarily 1. Thus, our goal is to define a distribution $\reference_x''$ that preserves Equations (\ref{ratio}) and (\ref{TV}). Let $\tau=\int_y \max\{\reference_x(y)-\reference_x'(y), 0\} dy$ and $\tau'=\int_y \max\{\reference_x'(y)-\reference_x(y), 0\} dy$, as in Figure~\ref{fig:PQ}. Note that $\tau=\tau'$ if and only if $\int_y \reference_x'(y) dy=1$. Thus there are two ways that $\reference_x'$ can fail to be a distribution.

Suppose first that $\tau> \tau'$,
then intuitively, we can convert $\reference_x'$ into a distribution by moving $\reference_x'$ closer to $\reference$ only in the region where 
$\reference>\reference_x$. That is, for all $z>0$, define \[\reference_x^z(y)= [\reference_x(y)]_{e^{-z}\reference(y)}^{e^{\eps}\reference(y)}\] and define a function $f(z)$ by \[f(z)=\int_y \max\{\reference_x^z(y)-\reference_x(y), 0\} dy=\int_{\reference(y)>\reference_x(y)} \max\{\reference_x(y), e^{-z}\reference(y)\}-\reference_x(y) dy.\] Now, $f(\eps)=\tau'<\tau$ and $f(0)=\TV(\reference, \reference_x)\ge\tau$. Since $f$ is continuous in $z$, by the intermediate value theorem, there exists $0\le\eps'<\eps$ such that $f(\eps')=\tau$. Since the distributions $\reference_x'$ and $\reference_x^{\eps'}$ agree on the region where $\reference_x>\reference_x'>\reference$, we have \[\int_y \max\{\reference_x(y)-\reference_x^{\eps'}(y), 0\} dy = \int_y \max\{\reference_x(y)-\reference_x'(y), 0\} dy = \tau = \int_y \max\{\reference_x^{\eps'}(y)-\reference_x(y), 0\} dy.\] This implies that $\reference_x^{\eps'}$ is a distribution. Equation~\eqref{ratio} still holds with $\reference_x^{\eps'}$ in place of $\reference_x'$ since $\eps'\le\eps$, and \[\int_y \max\{\reference_x(y)-\reference_x^{\eps'}(y), 0\} dy = \int_y \max\{\reference_x(y)-\reference_x^{\eps'}(y), 0\} dy = \int_y \max\{\reference_x(y)-\reference_x'(y), 0\} dy\le\delta.\] We can perform a similar operation if $\tau<\tau'$. Now, we let the output of $\lr'(x)$ be distributed according to $\reference_x^{\eps'}$.
\end{proof}

For the proof of Theorem~\ref{approxDP} we will need the following simple lemma about the hockey-stick divergence.
\begin{lemma}\cite[Lemma 3.17]{DR14-book}\label{tvdistancefromprivate}
Given random variables $P$, $Q$, $P'$ and $Q'$, if  $D_{e^{\eps}}(P',Q')\le\delta$, $\TV(P,P')\le\delta'$ and $\TV(Q,Q')\le\delta'$ then $D_{e^{\eps}}(P,Q)\le\delta+(e^{\eps}+1)\delta'$.
\end{lemma}

\rapproxDP*
\begin{proof}[Proof of Theorem~\ref{approxDP}]
Let $X_0$ and $X_1$ be neighboring datasets of size $n$ such that $x^0_1\neq x_1^1$. As in the proof of Theorem~\ref{shuffletobinoms}, we can assume without loss of generality that for all $j \in [2:n]$, $x_j \not\in\{x_1^0,x_1^1\}$.

Now fixing $i\in[n]$ and auxiliary input $z_{1:i-1}$, for $x\in\mathcal{D}$, let $\lr(x)~=~\Aldp[i](z_{1:i-1}, x)$.
If we define $\reference=\lr(x_1^0)$ then $\lr$ is a deletion $(\eps_0,\delta_0)$-DP local randomizer with reference distribution $\rho$.
By Lemma~\ref{approxtopure}, there exists a randomizer $\lr'$ that is a deletion $\eps_0$-DP local randomizer with the same reference distribution. In particular, $\lr(x_1^0)$ and $\lr'(x_1^1)$ are $\eps_0$-indistinguishable, and $\TV(\lr(x_1^1), \lr'(x_1^1))\le\delta_0$. Now, by Lemma \ref{lem:kov}, there exist distributions $\Q(x_1^0)$ and $\Q(x_1^1)$ so that
\begin{equation}\label{mix1}\lr(x_1^0)=\frac{e^{\eps_0}}{e^{\eps_0}+1}\Q(x_1^0)+\frac{1}{e^{\eps_0}+1}\Q(x_1^1),\end{equation} and \begin{equation}\label{mix2}\lr'(x_1^1)=\frac{1}{e^{\eps_0}+1}\Q(x_1^0)+\frac{e^{\eps_0}}{e^{\eps_0}+1}\Q(x_1^1).\end{equation}
Our next goal is to decompose $\lr(x)$ in terms of $\Q(x_1^0)$ and $\Q(x_1^1)$ for all $x \in \D$.
By convexity of the hockey-stick divergence, for any $x\in\mathcal{D}$, $\lr(x)$ is $(\eps_0, \delta_0)$-indistinguishable from $\frac{1}{2}(\lr(x_1^0)+\lr(x_1^1))$. That is, $\lr$ is a $(\eps_0,\delta_0)$ deletion DP local randomizer with reference distribution $\frac{1}{2}(\lr(x_1^0)+\lr(x_1^1))$. Therefore, by Lemma~\ref{approxtopure}, we can define $\lr''$ such that for all $x\in\mathcal{D}$, $\lr''(x)$ and $\frac{1}{2}(\lr(x_1^0)+\lr(x_1^1))$ are $\eps_0$-indistinguishable and $\TV(\lr''(x), \lr(x))\le\delta_0$. This implies that there exists a randomized algorithm $\lo \colon \D \to \cS$ such that we can decompose $\lr''(x)$ as \[\lr''(x) = \frac{1}{2e^{\eps_0}} \lr(x_1^0)+\frac{1}{2e^{\eps_0}} \lr(x_1^1)+ \left(1-\frac{1}{e^{\eps_0}}\right)\lo(x).\]
Next we define a randomizer $\mathcal{L}$ by $\mathcal{L}(x_1^0)=\lr(x_1^0)$, $\mathcal{L}(x_1^1)=\lr'(x_1^1)$ and for other $x\in\mathcal{D}$, \begin{align}
\mathcal{L}(x) &= \frac{1}{2e^{\eps_0}} \lr(x_1^0)+ \frac{1}{2e^{\eps_0}} \lr'(x_1^1)+  \left(1-\frac{1}{e^{\eps_0}}\right)\lo(x)\nonumber\\
&= \frac{1}{2e^{\eps_0}}\Q(x_1^0)+\frac{1}{2e^{\eps_0}}\Q(x_1^1)+ \left(1-\frac{1}{e^{\eps_0}}\right)\lo(x) \label{mix3}.
\end{align}

Note that $\TV(\lr(x_1^0), \mathcal{L}(x_1^0))=0$, $\TV(\lr(x_1^1), \mathcal{L}(x_1^1)) \leq \delta_0$ and for all $x\in\mathcal{D}\setminus\{x_1^0,x_1^1\}$, \[\TV(\lr(x), \mathcal{L}(x))\le \TV(\lr(x), \lr''(x))+\TV(\lr''(x), \mathcal{L}(x))\le \delta_0+ \frac{1}{2e^{\eps_0}}\TV(\lr'(x_1^1), \lr(x_1^1))\le \left(1+ \frac{1}{2e^{\eps_0}}\right)\delta_0.\]

For all $i\in[n]$ and $z_{1:i-1}\in\out[1]\times \cdots \times \out[i-1]$, let $\mathcal{L}^{(i)}(z_{1:i-1}, \cdot)$, $\Q^{(i)}(z_{1:i-1}, \cdot)$ and $\lo^{(i)}(z_{1:i-1}, \cdot)$ denote the result of this transformation applied to $\lr^{(i)}(z_{1:i-1},\cdot)$.
Let ${\A}_{\mathcal{L}}$ be the same algorithm as $\shuffler$ except $\lr^{(i)}(z_{1:i-1},x)$ is replaced by $\mathcal{L}^{(i)}(z_{1:i-1},x)$. Note that, $\shuffler$ applies each randomizer exactly once and hence, by the union bound, \[\TV(\shuffler(X_0), \A_{\mathcal{L}}(X_0))\le n(1+e^{-\eps_0}/2)\delta_0\;\;\; \text{and}\;\;\; \TV(\shuffler(X_1), \A_{\mathcal{L}}(X_1))\le n(1+e^{-\eps_0}/2)\delta_0.\] Thus, by Lemma~\ref{tvdistancefromprivate}, for any $\eps>0$, if $\A_\L(X_0)$ and $\A_\L(X_1)$ are $(\eps,\delta)$-indistinguishable then $\shuffler(X_0)$ and $\shuffler(X_1)$ are $(\eps, \delta+(e^{\eps}+1)(1+e^{-\eps_0}/2)n\delta_0)$-indistinguishable. So, all we need is to analyze the hockey-stick divergence between $\A_\L(X_0)$ and $\A_\L(X_1)$.

Equations~\eqref{mix1}, ~\eqref{mix2}, and~\eqref{mix3}
mirror Equations~\eqref{eq:decomp0}, ~\eqref{eq:decomp1} and ~\eqref{mixture} in the proof of Theorem~\ref{shuffletobinoms}, so the proof will proceed similarly. We define a randomizer $\L_\Q^{(i)}$ as follows:
For all $x\in \D$, $i\in[n]$ and values of $z_{1:i-1}$ we let
\[\L_\Q^{(i)}(z_{1:i-1}, x)=
\begin{cases} \Q^{(i)}(z_{1:i-1},x_1^0) & \text{if } x=x_1^0; \\
                                    \Q^{(i)}(z_{1:i-1},x_1^1) & \text{if } x=x_1^1;  \\
                                    \L^{(i)}(z_{1:i-1}, x) & \text{otherwise.} \end{cases}\]
Let $\A_\Q$ be defined in the same way as $\shuffler$, except $\Aldp[i]$ is replaced with $\L_\Q^{(i)}$.
Equations~\eqref{mix1} and~\eqref{mix2}, allow us to decompose $\A_\L(X_0)$ and $\A_\L(X_1)$ into the mixture of two components as follows:
\begin{equation}\label{shufflemix}\A_\L(X_0) = \frac{e^{\eps_0}}{e^{\eps_0}+1}\A_\Q(X_0)+\frac{1}{e^{\eps_0}+1}\A_\Q(X_1)\;\;\text{and}\;\;\A_\L(X_1) = \frac{1}{e^{\eps_0}+1}\A_\Q(X_0)+\frac{e^{\eps_0}}{e^{\eps_0}+1}\A_\Q(X_1).
\end{equation}

To compute the divergence between $\A_\Q(X_0)$ and $\A_\Q(X_1)$ we note that by Equation~\eqref{mix3}, for all $x\notin\{x_1^0,x_1^1\}$,
\[\L_\Q^{(i)}(z_{1:i-1}, x) = \frac{1}{2e^{\eps_0}}\L_\Q^{(i)}(z_{1:i-1}, x_1^0)+\frac{1}{2e^{\eps_0}}\L_\Q^{(i)}(z_{1:i-1}, x_1^1)+\left(1-\frac{1}{e^{\eps_0}}\right)\lo^{(i)}(z_{1:i-1},x).\] Therefore, by Lemma~\ref{binomstomixtures}, $\A_\Q(X_0)$ and $\A_\Q(X_1)$ can be seen as postprocessing of random variables $(A+1, C-A)$ and $(A,C-A+1)$, respectively, where $C\sim \bin(n-1, e^{-\eps_0})$ and $A\sim\bin(C, 1/2)$.

Finally, as before, Equation~\eqref{shufflemix} implies that ${\A}_{\mathcal{L}}(X_0)$ and ${\A}_{\mathcal{L}}(X_1)$ can be seen as postprocessing of random variables $(A+\Delta, C-A+1-\Delta)$ and $(A+1-\Delta,C-A+\Delta)$, respectively, where $\Delta \sim \Ber\left(\frac{e^{\eps_0}}{e^{\eps_0}+1}\right)$. By Lemma~\ref{binomials}, we get that
${\A}_{\mathcal{L}}(X_0)$ and ${\A}_{\mathcal{L}}(X_1)$ are $\left(\eps, \delta\right)$-indistinguishable where
\[\eps\le\log\left(1+\frac{e^{\eps_0}-1}{e^{\eps_0}+1}\left(\frac{8\sqrt{\log(4/\delta)}}{\sqrt{pn}}+\frac{8}{pn}\right)\right).
\]
\end{proof}
\section{Proof of Theorem~\ref{higheps0new}}\label{specificproof}

\rhighepsOnew*
\begin{proof}[Proof of Theorem~\ref{higheps0new}]
The proof follows the proof of Theorem~\ref{higheps0} with a modification to handle an additional mixture component.
Define a random variable $Y$ as follows \[Y = \begin{cases} 0 & \text{w.p.   } p\\ 1 & \text{w.p.   } p\\ 2 & \text{w.p.   } p \\ 3 & \text{w.p.   } 1-3p \end{cases}.\] Given a dataset $X_b$ for $b \in \zo$ we generate $n$ samples from $\zotf$ in the following way. Client number one (holding the first element of the dataset) samples $d$ from $\Ber(q)$. If $d=1$ it reports $b$. Otherwise, it reports 2. Clients $2,\ldots,n$ each report an independent sample from $Y$. We then shuffle the reports randomly. Let $\rho_b$ denote the resulting distribution over $\zotf^n$.

We claim that there exists a post-processing function $\Phi$ such that for $y$ sampled from $\rho_b$, $\Phi(y)$ is distributed identically to $\shuffler(X_b)$. The argument is essentially identical to the argument we used in Lemma~\ref{binomstomixtures} with the postprocessing done as described in Algorithm~\ref{postprocessing-general}.

To analyze the divergence between $\rho_0$ and $\rho_1$ we, again, observe that these are symmetric distributions. By definition, we know that $\rho_0$ is identical to the distribution of $(A+\Gamma, B, C+1-\Gamma,D)$, where $(A,B,C,D)\sim\mnom(n-1;p,p,p,1-3p)$ and $\Gamma \sim \Ber(q)$. Similarly, $\rho_1$ is identical to the distribution of $(A, B+\Gamma, C+1-\Gamma,D)$. Note that $D=n-1-A-B-C$ and therefore does not affect the divergence.

\begin{algorithm}
  \textbf{Input:} $x^0_1,x^1_1,x_2,\ldots, x_n$; $y \in \zotf^n$\\
  $J := \emptyset$\\
  \For{$i=1,\ldots,n$}{
  \If{$y_i =3$}{
  Let $j_i$ be a randomly and uniformly chosen element of $[2:n]\setminus J$\\
  $J := J \cup \{j_i\}$
  }
  Sample $z_i$ from $ \begin{cases} \qozi & \text{if } y_i=0; \\
                                    \qooi & \text{if } y_i=1;  \\
                                    \qzi & \text{if } y_i=2;  \\
                                    \lo^{(i)}(z_{1:i-1}, x_{j_i}) & \text{if } y_i=3. \end{cases}$\\
  }
  \textbf{return} $z_1, \ldots, z_n$
  \caption{Post-processing function, $\Phi$} \label{postprocessing-general}
\end{algorithm}

We now analyze the hockey-stick divergence between the distribution of $(A+\Gamma, B, C+1-\Gamma)$ and the distribution of $(A, B+\Gamma, C+1-\Gamma)$ which we denote by $\tau_0$ and $\tau_1$, respectively. We let $\kappa_0$ be the distribution of $(A,B,C+1)$ for $(A,B,C,D)\sim\mnom(n-1;p,p,p,1-3p)$, $\kappa_1^0$ be the distribution of $(A+1,B,C)$ and  $\kappa_1^1$ be the distribution of $(A,B+1,C)$. Then by definition we have that for $b\in \zo$,
\equ{\tau_b = (1-q) \kappa_0 + q \kappa_1^b .\label{eq:general-alg-mixture}}

We now upper-bound the divergence between $\kappa_0$ and $\kappa_1^b$, and between $\kappa_1^0$ and $\kappa_1^1$ by the divergence between the pair of distributions given in Lemma~\ref{binomials-aux}. For this purpose we restate the lemma with the notation of multinomial distributions and setting the parameter $p$ of the lemma to be $2p$. For $p\in(0,1/2]$ and $n\in\mathbb{N}$ such that $p~\ge~ \frac{8\ln(2/\delta)}{n}$ let $(A',B',C')~ \sim \mnom(n-1; p,p,1-2p)$. Then $P=(A'+1,B')$ and $Q=(A',B'+1)$ are $(\log(1+\frac{4\sqrt{2\log(4/\delta)}}{\sqrt{pn}}+\frac{4}{pn}), \delta)$- indistinguishable.

Given a sample $(a,b)\in [n]\times[n]$ we postprocess it by sampling $c \sim \Bin(n-a-b,p/(1-2p))$ and outputting $(a,b,c)$. This postprocesses a sample from $P$ into a sample from $\kappa_1^0$ and a sample from $Q$ into a sample from $\kappa_1^1$. By the postprocessing property of the hockey-stick divergence and Lemma~\ref{binomials-aux}, we get that $\kappa_1^0$ and $\kappa_1^1$  are \[\left(\log\left(1+\frac{4\sqrt{2\log(4/\delta)}}{\sqrt{pn}}+\frac{4}{pn}\right), \delta\right)\text{-indistinguishable.}\]
By sampling $c$ as before and outputting $(a,c,b)$ we can postprocess from $P$ and $Q$ to the pair $\kappa_1^0$ and $\kappa_0$.  Similarly, by outputting $(c,a,b)$  we can postprocess $P$ and $Q$ to the pair $\kappa_1^1$ and $\kappa_0$. Thus we obtain that for $b\in \zo$, $\kappa_0$ and $\kappa_1^b$ are \[\left(\log\left(1+\frac{4\sqrt{2\log(4/\delta)}}{\sqrt{pn}}+\frac{4}{pn}\right), \delta\right)\text{-indistinguishable.}\]

Now, Equation~\eqref{eq:general-alg-mixture} and advanced joint convexity of the hockey-stick divergence (Lemma~\ref{advancedjointconvexity}) imply that $\shuffler(X_0)$ and $\shuffler(X_1)$ are \[\left(\log\left(1+q\left(\frac{4\sqrt{2\log(4/\delta)}}{\sqrt{pn}}+\frac{4}{pn}\right)\right), q\delta\right)\text{-indistinguishable.}\]
\end{proof}

\section{Proof of Corollary~\ref{kRRthm}}\label{appendixkrr}

\rkRRthm*

\begin{proof}[Proof of Corollary~\ref{kRRthm}] To show that $\shuffler$ is $(\eps,\delta)$-DP, it is sufficient to show that for any pair of neighboring datasets $X_0$ and $X_1$, $\shuffler(X_0)$ and $\shuffler(X_1)$ are $(\eps,\delta)$-indistinguishable. Let $X_0$ and $X_1$ be a pair of neighboring datasets with $x_1^0\neq x_1^1$.
Let $\qzi$ be the uniform distribution on $[k]$, and for any $j\in[k]$, let $\indicator_j$ is the distribution that always outputs $j$. So for any $x\in\mathcal{D}$,
\[\Aldp[i](z_{1:i-1}, x)=\kRR(f^{(i)}(z_{1:i-1},x)) = \frac{k}{e^{\eps_0}+k-1}\qzi+\frac{e^{\eps_0}-1}{e^{\eps_0}+k-1}\indicator_{f^{(i)}(z_{1:i-1},x)}.\]
Let $p=\frac{k}{(k+1)(e^{\eps_0}+k-1)}$. Now, note that $\qzi = \frac{1}{k+1}\qzi+\frac{1}{k+1}\sum_{j=1}^k \indicator_{j}.$ So, for any $x$, if we let \[\lo^{(i)}(z_{1:i-1},x) = \frac{1}{(1-3p)}\left(\frac{k}{(k+1)(e^{\eps_0}+k-1)}\sum_{j\in [k],\ j\notin\{ f^{(i)}(z_{1:i-1},x_1^0), f^{(i)}(z_{1:i-1},x_1^1)\}} \indicator_j + \frac{e^{\eps_0}-1}{e^{\eps_0}+k-1}\indicator_{f^{(i)}(z_{1:i-1},x)}\right)\]
then
\begin{align*}
\Aldp[i](z_{1:i-1}, x) &= \frac{k}{e^{\eps_0}+k-1}\qzi+\frac{e^{\eps_0}-1}{e^{\eps_0}+k-1}\indicator_{f^{(i)}(z_{1:i-1},x)}\\
&= \frac{k}{e^{\eps_0}+k-1}\left(\frac{1}{k+1}\qzi+\frac{1}{k+1}\sum_{\substack{j=1}} ^k \indicator_j\right)+\frac{e^{\eps_0}-1}{e^{\eps_0}+k-1}\indicator_{f^{(i)}(z_{1:i-1},x)}\\
&= p\qzi+p\indicator_{f^{(i)}(z_{1:i-1},x_1^0)}+p\indicator_{f^{(i)}(z_{1:i-1},x_1^1)}+(1-3p)\lo_b^{(i)}(z_{1:i-1},x)
\end{align*}
Therefore, $\Aldp[i](z_{1:i-1}, x)$
satisfies the conditions of Theorem~\ref{higheps0new} with $q=\frac{e^{\eps_0}-1}{e^{\eps_0}+k-1}$ and $p=\frac{k}{(k+1)(e^{\eps_0}+k-1)}$. Note that $\eps_0\le\ln\left(\frac{n}{16\log(2/\delta)}\right)$ implies that $p\ge \frac{8\ln(2/\delta)}{n}$ so by Theorem~\ref{higheps0new}, $\shuffler(X_0)$ and $\shuffler(X_1)$ are $(\eps,\delta)$-indistinguishable where
\begin{align*}
\eps &\leq \log\left(1+\frac{e^{\eps_0}-1}{e^{\eps_0}+k-1}\left(\frac{4\sqrt{2(k+1)(e^{\eps_0}+k-1)\log(4/\delta)}}{\sqrt{kn}}+\frac{4(k+1)(e^{\eps_0}+k-1)}{kn}\right)\right)\\
&\le \log\left(1+(e^{\eps_0}-1)\left(\frac{4\sqrt{2(k+1)\log(4/\delta)}}{\sqrt{(e^{\eps_0}+k-1)kn}}+\frac{4(k+1)}{kn}\right)\right)\\
\end{align*}
\end{proof}

\section{Implementation of \texttt{Clones, empirical}}\label{pseudocode}

In this section we outline our implementation of the proof of Theorem~\ref{higheps0} to compute the amplification bound for general local randomizers. Let $q=\frac{e^{\eps_0}}{e^{\eps_0}+1}$ and $p=e^{-\eps_0}$. Recall that the shuffled $\eps$ is upper bounded by the divergence between the random variables
\begin{equation}\label{eq:PQ} P = \begin{cases} (A, C) & \text{w.p. } q\\(A+1,C) & \text{w.p. } 1-q \end{cases} \;\;\;\text{and} \;\;\; Q = \begin{cases} (A, C) & \text{w.p. } 1-q\\(A+1,C) & \text{w.p. } q \end{cases},
\end{equation} where $C\sim\Bin(n-1,p)$ and $A\sim\Bin(C,1/2)$. For a given $\delta$, our goal is to compute an approximately minimal $\eps$ such that $P$ and $Q$ are $(\eps,\delta)$-indistinguishable. It is computationally easier to compute an approximately minimal $\delta$ for a given $\eps$ than the converse, so for a given $\delta$ we'll use binary search to find such an $\eps$. Algorithm~\ref{binarysearch} takes as input a function $M$ that computes an approximation to the smallest $\delta$ such that $P$ and $Q$ are $(\eps,\delta)$-indistinguishable for a given $\eps$.

\begin{algorithm}\DontPrintSemicolon
  \textbf{Input:} $\eps_0, \delta, T, M$\\
  $\eps^L = 0$\\
  $\eps^R = \eps_0$\\
  \For{$t\in[T]$}{
  $\eps_t = \frac{\eps^L+\eps^R}{2}$\\
  $\delta_t = M(\eps_t)$ \tcp*{$M(\eps_t)=$ the smallest value such that $P$ and $Q$ are $(\eps_t, \delta_t)$-indistinguishable}
  \If{$\delta_t<\delta$}{$\eps^R = \eps_t$}
  \Else{$\eps^L = \eps_t$}
  }
  {\bf return} $\eps^R$
  \caption{Binary Search, \texttt{BinS}} \label{binarysearch}
\end{algorithm}

Next, we need to design the algorithm $M$. Note that for a given $\eps$, the minimal $\delta$ is given by the equation \[\delta = D_{e^{\eps}}(P,Q)= \max\left\{ \int_{(a,c)} \max\{0,P(a,c)-e^{\eps}Q(a,c)\} d(a,c), \int_{(a,c)} \max\{0,Q(a,c)-e^{\eps}P(a,c)\} d(a,c)\right\},\] where for ease of notation we use $P$ and $Q$ for both the random variables and their probability density functions (pdf). We present an algorithm $M$ that upper bounds this integral. Let us first look at an important subroutine: for a given $c$, Algorithm~\ref{partint} computes $\int_{a} \max\{0, P(a,c)-e^{\eps}Q(a,c)\}da$ if $b=+$ and $\int_{a} \max\{0,Q(a,c)-e^{\eps}P(a,c)\}da$ if $b=-$.

\begin{algorithm}\DontPrintSemicolon
  \textbf{Input:} $c, \eps, \eps_0, b$\\
  $q = \frac{e^{\eps_0}}{e^{\eps_0}+1}$\\
  $\eps_{q,\eps} = \ln\left(\frac{(e^{\eps}+1)q-1}{(e^{\eps}+1)q-e^{\eps}}\right)$\\
  \If{$b=+$}{
  $\beta = \frac{1}{e^{\eps_{q,\eps}}+1}$}
  \Else{
  $\beta = \frac{1}{e^{-\eps_{q,\eps}}+1}$}
  $\tau = \beta (c+1)$\\
  $\gamma_P = q*\Pr(\Bin(c,0.5)\le \tau)+(1-q)\Pr(\Bin(c,0.5)\le\tau-1)$\\
  $\gamma_Q = (1-q)*\Pr(\Bin(c,0.5)\le \tau)+q\Pr(\Bin(c,0.5)\le\tau-1)$\\
  \If{$b=+$}{\textbf{return} $\gamma_P-e^{\eps}\gamma_Q$}
  \Else{\textbf{return} $(1-\gamma_Q)-e^{\eps}(1-\gamma_P)$}
  \caption{$\onedint$} \label{partint}
\end{algorithm}

\begin{lemma}\label{algosubroutine}
For $c\in\mathbb{Z}$, and with $P$ and $Q$ as in Equation~\eqref{eq:PQ}, if $\eps<\eps_0$ then \[\Pr(\Bin(n-1,p)=c)\cdot\onedint(c,\eps,\eps_0, +) =\int_{\mathbb{Z}} \max\{0,P(a,c)-e^{\eps}Q(a,c)\} da\] and \[\Pr(\Bin(n-1,p)=c)\cdot\onedint(c,\eps,\eps_0, -) =\int_{\mathbb{Z}} \max\{0,Q(a,c)-e^{\eps}P(a,c)\} da.\]
\end{lemma}

\begin{proof}
First, let us characterize when $P(a,c)-e^{\eps}Q(a,c)>0$. We will let $C$ denote both the random variable $\Bin(n-1,p)$ and its pdf. Similarly, $A_c$ denotes the random variable $\Bin(c,1/2)$ and its pdf. 
\begin{align*}
    P(a,c)-e^{\eps}Q(a,c)>0 &\iff C(c)((q A_c(a) + (1-q) A_c(a-1))-e^{\eps}((1-q) A_c(a) + q A_c(a-1)))>0\\
    &\iff (q-e^{\eps}(1-q))A_c(a)>(e^{\eps}q-(1-q))A_c(a-1)\\
\end{align*}
Now, if $\eps<\eps_0$ then $q-e^{\eps}(1-q)> 0$.
Let $\eps_{q,\eps}=\ln\left(\frac{e^{\eps}q-(1-q)}{q-e^{\eps}(1-q)}\right)$. So,
\begin{align*}
    P(a,c)-e^{\eps}Q(a,c)>0 &\iff \binom{c}{a}>e^{\eps_{q,\eps}}\binom{c}{a-1}\\
    &\iff \frac{c+1-a}{a} > e^{\eps_{q,\eps}}\\
    &\iff (e^{\eps_{q,\eps}}+1)a<c+1\\
    &\iff a<\frac{c+1}{e^{\eps_{q,\eps}}+1}
\end{align*}
If $b=+$ then $\tau=\frac{c+1}{e^{\eps_{q,\eps}}+1}$ so $P(a,c)-e^{\eps}Q(a,c)>0\iff a<\tau$. Note that $\gamma_P = \Pr(P\le\tau)$ and $\gamma_Q = \Pr(Q\le\tau)$. Therefore, \[\int_{\mathbb{Z}} \max\{0,P(a,c)-e^{\eps}Q(a,c)\} da = \int_{-\infty}^{\tau} (P(a,c)-e^{\eps}Q(a,c)) da = C(c)(\gamma_P-e^{\eps}\gamma_Q)\]

\noindent Now, for the second integral,
\begin{align*}
    Q(a,c)-e^{\eps}P(a,c)>0 &\iff C(c)(((1-q) A_c(a) + q A_c(a-1))-e^{\eps}(q A_c(a) + (1-q) A_c(a-1)))>0\\
    &\iff ((1-q)-e^{\eps}q)A_c(a)>(e^{\eps}(1-q)-q)A_c(a-1)
\end{align*}
For any $\eps$, $(1-q)-e^{\eps}q<0$ so,
\begin{align*}
    Q(a,c)-e^{\eps}P(a,c)>0 &\iff A_c(a)<e^{-\eps_{q,\eps}}A_c(a-1)\\
    &\iff \frac{c+1-a}{a}<e^{-\eps_{q,\eps}}\\
    &\iff (e^{-\eps_{q,\eps}}+1)a>c+1\\
    &\iff a>\frac{c+1}{e^{-\eps_{q,\eps}}+1}
\end{align*}
If $b=-$ then $\tau=\frac{c+1}{e^{-\eps_{q,\eps}}+1}$ so $Q(a,c)-e^{\eps}P(a,c)>0\iff a>\tau$. Noting that $1-\gamma_P=\Pr(P\ge\tau)$ and $1-\gamma_Q=\Pr(Q\ge\tau)$ we have, \[\int_{\mathbb{Z}} \max\{0,Q(a,c)-e^{\eps}P(a,c)\} da = \int_{\tau}^{\infty} Q(a,c)-e^{\eps}P(a,c) da = C(c)((1-\gamma_Q)-e^{\eps}(1-\gamma_P).\]
\end{proof}

\begin{algorithm}\DontPrintSemicolon
  \textbf{Input:} $n, \eps_0, \delta^U, S, \eps$\\
  $\delta_P^0 = 0$ \tcp*{will keep track of $\int \max\{P(x)-e^{\eps}Q(x), 0\} dx$}
  $\delta_Q^0 = 0$ \tcp*{will keep track of $\int \max\{Q(x)-e^{\eps}P(x), 0\} dx$}
  $\zeta_C^0 = 0$ \tcp*{will keep track of the probability mass of $C$ covered}
  $p = e^{-\eps_0}$\\
  $T=\lfloor n/S \rfloor$\\
  \For{$t = \{0, \cdots, T \}$}{
  $B^t=t*S$\\
  \vspace{0.2in}
  \textbf{Step 1:} if $\max(\delta_P^{t}, \delta_Q^{t})>\delta^U$ then $P$ and $Q$ are not $(\eps, \delta^U)$-indistinguishable so exit.\\
  \If{$\max(\delta_P^{t}, \delta_Q^{t})>\delta^U$ }{\textbf{return} $\delta^U$}
  \vspace{0.2in}
  \textbf{Step 2:} any further contribution to either $\delta_P$ or $\delta_Q$ will not exceed $1-\zeta_C$, so if this is small we exit.\\
  \ElseIf{$1-\zeta_C^t<\delta_P^t$ and $1-\zeta_C^t<\delta_Q^t$}{\textbf{return} $\max\{\delta_P^t+1-\zeta_C^t, \delta_Q^t+1-\zeta_C^t\}$}
  \vspace{0.2in}
  \textbf{Step 3:} we estimate the contribution to $\delta_P$ and $\delta_Q$ from the next interval $[B^t, B^t+S)$\\
  \Else{
  $C_{\max} = B^t+S-1$\\
  $C_{\min} = B^t$\\
  $Pr_{[C_{\min}, C_{\max}]} = \Pr(\Bin(n-1,p)\in[C_{\min}, C_{\max}])$\\
  \vspace{0.2in}
  \textbf{Step 3a:} Compute contribution to $\delta_P$\\
  $c_{P, \max, \delta} = \onedint(C_{\max}, n, p, \eps, \eps_0, +)$\\
$c_{P, \min, \delta} = \onedint(C_{\min}, n, p, \eps, \eps_0, +)$\\
$\delta_P^{t+1} = \delta_P^t+Pr_{[C_{\min}, C_{\max}]}\cdot \max\{c_{P, \max, \delta}, c_{P, \min, \delta}\}$\\
\vspace{0.2in}
\textbf{Step 3b:} Compute contribution to $\delta_Q$\\
$c_{Q, \max, \delta} = \onedint(C_{\max}, n, p, \eps, \eps_0, -)$\\
$c_{Q, \min, \delta} = \onedint(C_{\min}, n, p, \eps, \eps_0, -)$\\
$\delta_Q^{t+1} = \delta_Q^t+\Pr_{[C_{\min}, C_{\max}]}\cdot \max\{c_{Q, \max, \delta}, c_{Q, \min, \delta}\}$\\
\vspace{0.2in}
\textbf{Step 3c:} Compute contribution to $\zeta_C$\\
$\zeta_C^{t+1} = \zeta_C^{t}+\Pr_{[C_{\min}, C_{\max}]}$
  }
  }
  \textbf{return} $\max\{\delta_P^{T+1},\delta_Q^{T+1}\}$
  \caption{$M_1$} \label{alg:ampbound}
\end{algorithm}

Lemma~\ref{algosubroutine} shows that for a fixed $c$, $\onedint$ computes the integral $\onedint(c,\eps,\eps_0, +)=\int_{a} \max\{0, P(a,c)-e^{\eps}Q(a,c)\}da$. Recall, our goal is to estimate \[\int_c\int_{a} \max\{0, P(a,c)-e^{\eps}Q(a,c)\}dadc=\int_c\onedint(c,\eps,\eps_0, +)dc.\] We could compute $\onedint(c,\eps,\eps_0, +)$ for every $c$, but in order to make the computation more efficient, instead of computing $\onedint(c,\eps,\eps_0, +)$ for every $c$, Algorithm~\ref{alg:ampbound} defines a parameter $S$, and only computes it for every $S$th value of $c$. For values between $c$ and $c+S$, we can leverage the fact that $\onedint(c,\eps,\eps_0, +)$ is monotone to bound their contribution to the integral.

\begin{lemma}\label{ismonotone}
For any $\eps<\eps_0$ and $b\in\{+,-\}$, if $c>c'$ then $\onedint(c,\eps,\eps_0, b)<\onedint(c',\eps,\eps_0,b)$.
\end{lemma}

\begin{proof} Let us prove the result for $b=+$, the proof for $b=-$ is identical. Recall $q=\frac{e^{\eps_0}}{e^{\eps_0}+1}$ and $c>0$, let \[P_c = \begin{cases} \Bin(c,1/2) & \text{w.p. } q \\ \Bin(c,1/2)+1 & \text{w.p. } 1-q\end{cases} \;\; \text{ and } \;\; Q_c = \begin{cases} \Bin(c,1/2) & \text{w.p. } 1-q \\ \Bin(c,1/2)+1 & \text{w.p. } q\end{cases}.\] Now, suppose that $c>c'$ then $P_{c} = P_{c'}+\Bin(c-c',1/2)$ and $Q_{c} = Q_{c'}+\Bin(c-c',1/2)$, where we are summing the random variables (not the pdfs). 
That is, $P_c$ and $Q_c$ can be obtained by post-processing $P_{c'}$ and $Q_{c'}.$ Also $\int_{\mathbb{Z}}\max\{0,P_c(a)-e^{\eps}Q_c(a)\}da$ is the hockey-stick divergence between $P_c$ and $Q_c$. Therefore, since the hockey-stick divergence satisfies the data processing inequality and by Lemma~\ref{algosubroutine}, \[\mathcal{B}(c,\eps,\eps_0,b) = \int_{\mathbb{Z}}\max\{0,P_c(a)-e^{\eps}Q_c(a)\}da\le \int_{\mathbb{Z}}\max\{0,P_{c'}(a)-e^{\eps}Q_{c'}(a)\}da = \mathcal{B}(c',\eps,\eps_0,b).\]
\end{proof}

Additionally, the algorithm takes in a value $\delta^U$ (set to be the target $\delta$ when called from binary search) so that if the current guess for $\eps$ is too small, then we may save on computation by aborting as soon as we have established that $\delta < \delta^U$. We remark that while we have stated the algorithm as iterating over the values of $t$ starting at zero, the proof does not require that the values $\{0,\ldots, T\}$ be processed in increasing order. In an actual implementation, we process these in increasing order of $|t-T/2|$ so as to process the $c$'s which have large probability mass first.

\begin{prop} For any $n\in\mathbb{N},\eps_0>0,\delta\in[0,1], T\in\mathbb{N}$, let $\eps = \BinS(\eps_0,\delta,T,M_1(n, \eps_0,\delta,S,\cdot))$ be the output of Algorithm~\ref{binarysearch} where $M_1$ is given by Algorithm~\ref{alg:ampbound}. Then $P$ and $Q$, as given by Equation~\eqref{eq:PQ}, are $(\eps,\delta)$-indistinguishable.
\end{prop}

\begin{proof} Note that since $P$ and $Q$ are $(\eps_0,\delta)$-indistinguishable and Algorithm~\ref{binarysearch} outputs the largest value inside the final range $[\eps_L, \eps_R],$ it suffices to show that at each iteration, \[\delta_t = M_1(n,\eps_0,\delta,S,\eps_t) \ge \min\left\{\delta, \max\left\{\int_{\mathbb{Z}}\int_{\mathbb{Z}} P(a,c)-e^{\eps_t}Q(a,c) d(a,c), \int_{\mathbb{Z}}\int_{\mathbb{Z}}Q(a,c)-e^{\eps_t}P(a,c) d(a,c)\right\}\right\}.\] That is, that Algorithm~\ref{binarysearch} makes conservative choices at each iteration. Let us then focus on showing that for any choices of $n,\eps_0,\eps,\delta^U,S$ such that $\eps_0>\eps$, \begin{equation}\label{conservative} M_1(n,\eps_0,\delta^U,S,\eps_t) \ge \min\left\{\delta^U, \max\left\{\int_{\mathbb{Z}}\int_{\mathbb{Z}} P(a,c)-e^{\eps_t}Q(a,c) d(a,c), \int_{\mathbb{Z}}\int_{\mathbb{Z}}Q(a,c)-e^{\eps_t}P(a,c) d(a,c)\right\}\right\}.
\end{equation} 

Let $t\in\{0,\cdots,T\}$ and $B^t=t*S$. Since 
$\mathcal{B}(\cdot,\eps,\eps_0,b)$ is monotone, $\mathcal{B}(B^t,\eps,\eps_0,b)\ge \mathcal{B}(D, \eps,\eps_0,b)$ for any $D\in[B^t,B^t+S)$. Thus, \[Pr_{[C_{\min}, C_{\max}]}\cdot \max\{\mathcal{B}(B^t,\eps,\eps_0,+), \mathcal{B}(B^t+S-1,\eps,\eps_0,+)\}\ge \int_{B^t}^{B^t+S-1} \int_{\mathbb{Z}} \max\{0, P(a,c)-e^{\eps}Q(a,c)\} d(a,c)\] and \[Pr_{[C_{\min}, C_{\max}]}\cdot \max\{\mathcal{B}(B^t,\eps,\eps_0,-), \mathcal{B}(B^t+S-1,\eps,\eps_0,-)\}\ge \int_{B^t}^{B^t+S-1} \int_{\mathbb{Z}} \max\{0, Q(a,c)-e^{\eps}P(a,c)\} d(a,c).\]
Let $\mathcal{R}_C^t=[0,B^t)$ be the range of values of $C$ that have already been covered by the first $t-1$ iterations. So, the above equations imply that we always have \[\delta_P^t\ge \int_{\mathcal{R}_C^t} \int_{\mathbb{Z}} \max\{0, P(a,c)-e^{\eps}Q(a,c)\} d(a,c)\;\; \text{and} \;\;\delta_Q^t \ge \int_{\mathcal{R}_C^t} \int_{\mathbb{Z}} \max\{0, Q(a,c)-e^{\eps}P(a,c)\} d(a,c).\] Now, there are three ways Algorithm~\ref{alg:ampbound} can terminate.

Case 1: If at round $t$, $\max\{\delta_P^t,\delta_Q^t\}>\delta^U$ then $M_1(n,\eps_0,\delta,S,\eps_t)=\delta^U$ and Equation~\eqref{conservative} holds.

Case 2: If at round $t$, $1-\zeta_C^t<\delta_P^t$ and $1-\zeta_C^t<\delta_Q^t$ then
\begin{align*}
\int_{\mathbb{Z}}\int_{\mathbb{Z}} &\max\{0, P(a,c)-e^{\eps}Q(a,c)\} d(a,c)\\ &= \int_{\mathcal{R}_C^t}\int_{\mathbb{Z}} \max\{0, P(a,c)-e^{\eps}Q(a,c)\} d(a,c) + \int_{\mathbb{Z}\backslash\mathcal{R}_C^t}\int_{\mathbb{Z}} \max\{0, P(a,c)-e^{\eps}Q(a,c)\} d(a,c)\\
&\le \delta_P+\Pr(\Bin(n-1,p)\in\mathbb{Z}\backslash\mathcal{R}_C^t)\\
&\le \delta_P^t+1-\zeta_C^t,
\end{align*} and similarly \[\int_{\mathbb{Z}}\int_{\mathbb{Z}} \max\{0, Q(a,c)-e^{\eps}P(a,c)\} d(a,c)\le \delta_Q^t+1-\zeta_C^t. \] Since $M_1(n,\eps_0,\delta,S,\eps_t)=\max\{\delta_P^t+1-\zeta_C^t, \delta_Q^t+1-\zeta_C^t\}$, this implies Equation (\ref{conservative}) holds.

Case 3: If the algorithm doesn't terminate early, then $\mathcal{R}_C^{T+1}=\mathbb{Z}$ so \[\int_{\mathbb{Z}}\int_{\mathbb{Z}} \max\{0, P(a,c)-e^{\eps}Q(a,c)\} d(a,c)\le\delta_P^{T+1}\] and
\[\int_{\mathbb{Z}}\int_{\mathbb{Z}} \max\{0, Q(a,c)-e^{\eps}P(a,c)\} d(a,c)\le\delta_Q^{T+1}.\]
Since $M_1(n,\eps_0,\delta,S,\eps_t)=\max\{\delta_P^{T+1}, \delta_Q^{T+1}\}$ this implies Equation~\eqref{conservative} holds.
\end{proof}

\section{Implementation of \texttt{2RR, lower bound}}\label{2RRimplementation}

For any $\eps_0 >0$, binary randomized response $\2RR\colon \zo\to\zo$ is defined as \[\2RR(x) = \begin{cases} x & \text{with probability } \frac{e^{\eps_0}}{e^{\eps_0}+1}\\  1-x & \text{with probability } \frac{1}{e^{\eps_0}+1}\end{cases}.\]  Let $\shuffler:\zo^n\to\zo^n$ be the algorithm that given a dataset $x_{1:n}\in\zo^n$, samples a uniformly random permutation $\pi$,  computes $z_i=\2RR(x_{\pi(i)})$ for all $i\in[n]$, then outputs $z_{1:n}$. That is, each client simply reports their value using $\2RR$, and the reports are permuted.

For any $\delta\in[0,1]$, and random variables $P$ and $Q$, let $D_{\infty}^{\delta}(P,Q)$ be the minimal $\eps$ such that $P$ and $Q$ are $(\eps,\delta)$-indistinguishable. Let $\eps_{\delta}$ be the minimal $\eps$ such that $\shuffler$ is $(\eps,\delta)$-DP so \[\eps_{\delta} = \max_{X_0,X_1} D_{\infty}^{\delta}(\shuffler(X_0),\shuffler(X_1)),\] where the maximum is over all possible pairs of neigboring datasets $X_0,X_1\in\zo^n$.
In the implementation of \texttt{2RR, lower bound}, we set \[X_0 = (0,\ldots,0) \text{ and } X_1=(1,0,\ldots,0)\] and compute a lower bound on $D_{\infty}^{\delta}(\shuffler(X_0),\shuffler(X_1))$, which in turn, gives a lower bound on $\eps_{\delta}$.

Again, since it is computationally easier to compute the minimal $\delta$ for a given $\eps$, we use binary search to find a lower bound on the minimal $\eps$ for a given $\delta$. Since we want a lower bound, we use Algorithm~\ref{binarysearchlower}, which is the same as Algorithm~\ref{binarysearch}, except at the final stage it outputs $\eps_L$ rather than $\eps_R$. The final component we need to describe is the function $M$, which given $\eps$, computes $D_{e^{\eps}}(\shuffler(X_0), \shuffler(X_1))$. Note that the output of $\shuffler$ is characterized by simply the number of 0s and 1s in the local reports. Thus, the divergence between $\shuffler(X_0)=z_{1:n}^0$ and $\shuffler(X_1)=z_{1:n}^1$ is the same as the divergence between \[c^0 = \sum_{i=1}^n z_i^0 = \Bin\left(n,\frac{1}{e^{\eps_0}+1}\right) \;\text{and}\; c^1=\sum_{i=1}^n z_i^1=\Bin\left(n-1,\frac{1}{e^{\eps_0}+1}\right)+\Ber\left(\frac{e^{\eps_0}}{e^{\eps_0}+1}\right),\] where $\Ber\left(\frac{e^{\eps_0}}{e^{\eps_0}+1}\right)$ denotes a Bernoulli random variable, and we are summing the random variables (NOT the pdfs). Since $c^0$ and $c^1$ are discrete, provided $n$ is not too large, we can efficiently compute \begin{equation}\label{Mfunction}
M(\eps)=D_{e^{\eps}}(c_0,c_1) = \max\left\{\sum_{c=0}^n \max\{c_0(c)-e^{\eps}c_1(c),0\}, \sum_{c=0}^n \max\{c_1(c)-e^{\eps}c_0(c),0\}\right\},
\end{equation} by explicitly computing the pdfs of $c^0$ and $c^1$. Now, since this computation is exact (up to numerical precision), for all $t\in[T]$, Algorithm~\ref{binarysearchlower} moves in the right direction at every iterate. This implies that for the duration of the algorithm $D_\infty^\delta(c_0,c_1)\in[\eps^L, \eps^R]$, before finally outputting the lower bound $\eps^L$.
\begin{prop}
For any $n\in\mathbb{N},\eps_0>0,\delta\in[0,1], T\in\mathbb{N}$, let $\eps = \texttt{BinSLower}(\eps_0,\delta,T,M(\cdot))$ be the output of Algorithm~\ref{binarysearchlower} where $M$ is given by Equation~\eqref{Mfunction}. Then $\eps_{\delta}\ge\eps$.
\end{prop}

\begin{algorithm}\DontPrintSemicolon
  \textbf{Input:} $\eps_0, \delta, T, M$\\
  $\eps^L = 0$\\
  $\eps^R = \eps_0$\\
  \For{$t\in[T]$}{
  $\eps_t = \frac{\eps^L+\eps^R}{2}$\\
  $\delta_t = M(\eps_t)$ \tcp*{$M(\eps_t)=$ the smallest value such that $P$ and $Q$ are $(\eps_t, \delta_t)$-indistinguishable}
  \If{$\delta_t<\delta$}{$\eps^R = \eps_t$}
  \Else{$\eps^L = \eps_t$}
  }
  {\bf return} $\eps^L$
  \caption{Binary Search, \texttt{BinSLower}} \label{binarysearchlower}
\end{algorithm}

\end{document}